\documentclass{article} 
\PassOptionsToPackage{numbers,compress}{natbib}
\usepackage[ruled,linesnumbered,vlined]{algorithm2e}\SetCommentSty{textnormal}
\usepackage{amsmath,amssymb,amsthm,mathtools}
\usepackage{booktabs}
\usepackage[dvipsnames]{xcolor}
\usepackage{url}
\usepackage{afterpage}
\usepackage{thmtools}
\usepackage{thm-restate}
\theoremstyle{plain}
\newtheorem{theorem}{Theorem}[section]
\newtheorem{proposition}[theorem]{Proposition}
\newtheorem{lemma}[theorem]{Lemma}
\newtheorem{corollary}[theorem]{Corollary}
\theoremstyle{definition}
\newtheorem{definition}[theorem]{Definition}
\newtheorem{assumption}[theorem]{Assumption}
\theoremstyle{remark}
\newtheorem{remark}[theorem]{Remark}
\usepackage{subcaption}
\usepackage{hyperref}
\usepackage{booktabs}
\usepackage{multirow}
 \usepackage{natbib}
\usepackage{adjustbox}
\usepackage{tabularx}
\usepackage{array}
\usepackage{placeins}
\newcolumntype{Y}{>{\raggedright\arraybackslash}X}

\hypersetup{
    colorlinks=true,
    citecolor=NavyBlue,
    linkcolor=NavyBlue,
    urlcolor=NavyBlue,
}
\usepackage[capitalise,nameinlink]{cleveref}

\usepackage{amsmath,amsfonts,bm}

\def\eqref#1{equation~\ref{#1}}

\def\1{\bm{1}}

\def\vone{{\bm{1}}}

\def\mC{{\bm{C}}}

\def\mI{{\bm{I}}}

\def\mW{{\bm{W}}}

\DeclareMathAlphabet{\mathsfit}{\encodingdefault}{\sfdefault}{m}{sl}
\SetMathAlphabet{\mathsfit}{bold}{\encodingdefault}{\sfdefault}{bx}{n}

\newcommand{\R}{\mathbb{R}}

\newcommand{\Var}{\mathrm{Var}}

\DeclareMathOperator*{\argmax}{arg\,max}
\DeclareMathOperator*{\argmin}{arg\,min}

\DeclareMathOperator{\diag}{diag}

\usepackage[most]{tcolorbox}

\newtcolorbox{tmplbox}[1]{
  enhanced, breakable, sharp corners, boxrule=0.6pt,
  colback=black!4, colframe=black!55,
  coltitle=black!55, fonttitle=\bfseries\scriptsize,
  attach boxed title to top left={yshift=-1.5mm,xshift=2mm},
  boxed title style={colback=white,boxrule=0pt,left=1mm,right=1mm,top=0.5mm,bottom=0.5mm},
  title={#1},
  fontupper=\ttfamily\scriptsize,
  top=4mm, left=3mm, right=2mm, bottom=2mm
}
\newtcolorbox{promptbox}[1]{
  enhanced, breakable, sharp corners, boxrule=0.6pt,
  colback=blue!5, colframe=blue!45!black,
  coltitle=blue!45!black, fonttitle=\bfseries\scriptsize,
  attach boxed title to top left={yshift=-1.5mm,xshift=2mm},
  boxed title style={colback=white,boxrule=0pt,left=1mm,right=1mm,top=0.5mm,bottom=0.5mm},
  title={#1},
  fontupper=\ttfamily\scriptsize,
  top=4mm, left=3mm, right=2mm, bottom=2mm
}
\newtcolorbox{outputbox}[1]{
  enhanced, breakable, sharp corners, boxrule=0.6pt,
  colback=green!5, colframe=green!35!black,
  coltitle=green!35!black, fonttitle=\bfseries\scriptsize,
  attach boxed title to top left={yshift=-1.5mm,xshift=2mm},
  boxed title style={colback=white,boxrule=0pt,left=1mm,right=1mm,top=0.5mm,bottom=0.5mm},
  title={#1},
  fontupper=\ttfamily\scriptsize,
  top=4mm, left=3mm, right=2mm, bottom=2mm
}

\usepackage{arxiv}

\usepackage[utf8]{inputenc} 
\usepackage[T1]{fontenc}    
\usepackage{hyperref}       
\usepackage{url}            
\usepackage{booktabs}       
\usepackage{amsfonts}       
\usepackage{nicefrac}       
\usepackage{microtype}      
\usepackage{lipsum}
\usepackage{graphicx}
\graphicspath{ {./images/} }

\title{Social Networks of LLM Agents}

\author{
 Kaixuan Liu \\
  Department of Computer Science\\
  Emory University\\
  Atlanta, GA 30322 \\
  \texttt{kaixuan.liu@emory.edu} \\
   \And
 Guojun Xiong\thanks{Correspondence to Guojun Xiong~<gjxiong@sjtu.edu.cn>.} \\
  School of Computer Science\\
  Shanghai Jiao Tong University\\
  Shanghai, China \\
  \texttt{gjxiong@sjtu.edu.cn} \\
  \AND
 Weinan Zhang \\
  School of Computer Science\\
  Shanghai Jiao Tong University\\
  Shanghai, China \\
  \texttt{wnzhang@sjtu.edu.cn} \\
  \And
 Shengpu Tang \\
  Department of Computer Science\\
  Emory University\\
  Atlanta, GA 30322 \\
  \texttt{shengpu.tang@emory.edu} \\
}

\begin{document}
\date{}
\maketitle

\begin{abstract}
Large language model (LLM) agents are increasingly deployed in interacting
populations, raising the question of what such populations come to believe
collectively. Whether a population aggregates genuine knowledge or collapses
into a false consensus directly affects how much such systems can be trusted.
Classical social-network models assume that the network itself determines how
beliefs combine. This assumption breaks down for LLM agents, whose limited
attention takes in only part of what they are exposed to, so these models
overstate how much information a population actually pools and cannot tell
genuine consensus from herding. We introduce \textsc{SNLA}, a framework that
models how much each agent actually influences others, rather than merely how
the network connects them. This influence depends on each agent's position in
the network and on how sharply attention focuses. Theoretically, we show on a
tractable proxy that narrow attention causes herding, where the effective sample
size stays bounded regardless of population size, while wide attention recovers
wisdom-of-crowds behavior only when the exposure graph is undirected and
degree-regular. Empirically, a controlled testbed validates these predictions
directly, and the herding--wisdom transition reproduces on operator-controlled
variants of three multi-agent LLM benchmarks. 
\end{abstract}

\section{Introduction}

\afterpage{\begin{figure*}[!t]
  \centering
  \includegraphics[width=\textwidth]{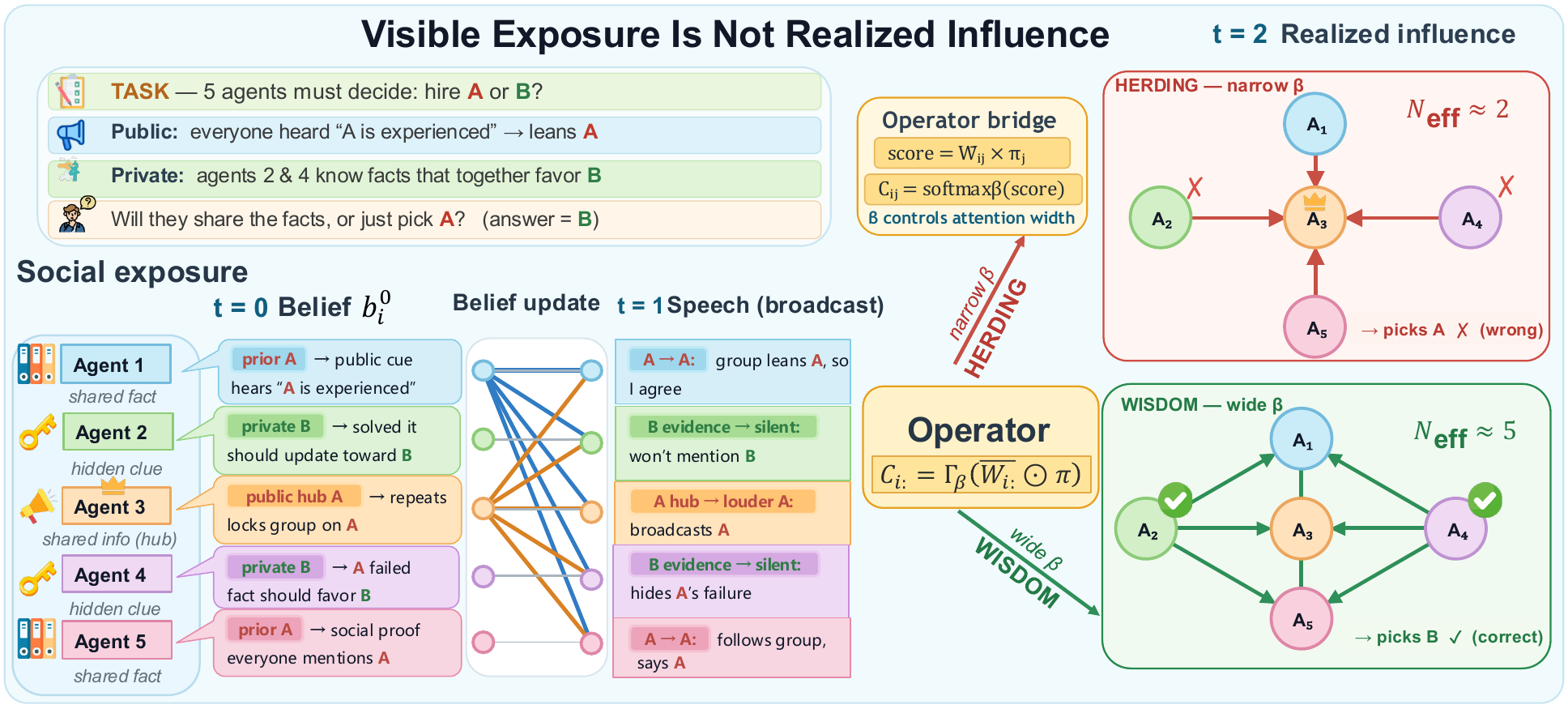}
  \caption{\textbf{From exposure to realized influence.} The visible
  \emph{exposure} graph $\bar{W}$ (left) records what each agent could
  see; edges are exposure weights and node $a_3$ carries high social
  power $\pi$. The realized-influence operator
  $C_{i\cdot}=\Gamma_\beta(\bar{W}_{i\cdot}\odot\pi)$ reweights each
  reader's exposure by source social power and concentrates it through a
  finite attention width $\beta$. The graph $C$ (right) depends
  sharply on $\beta$: narrow attention ($\beta{=}0.3$) collapses
  influence onto the most powerful source (\emph{herding}); wide
  attention ($\beta{=}3$) spreads it across the population (\emph{wisdom
  of crowds}); $\beta{=}1$ allocates proportionally.}
  \label{fig:framework}
\end{figure*}}

Large language model (LLM) agents are increasingly deployed in interacting
populations. They debate to reach answers \citep{du2024improving,
chan2024chateval, liang2024encouraging}, advise and persuade one another
\citep{park2023generative}, and simulate social processes at scale
\citep{gao2023s3, piao2025agentsociety}. In all of these settings, one
question recurs: what does the population come to believe collectively, and can
that collective belief be trusted? (i) A group that pools its members'
independent information can be more accurate than any individual
\citep{galton1907vox}; (ii) a group that copies its most visible members can
collapse into a confident error \citep{banerjee1992simple,
bikhchandani1992theory}. Distinguishing the two is essential before such systems
are relied upon for decisions.

The classical answer comes from the theory of social learning on networks. In the
DeGroot model \citep{degroot1974reaching} and its variants
\citep{friedkin1990social}, each agent repeatedly replaces its belief with a
weighted average of its neighbors', with weights read directly off the network:
the graph structure alone determines how beliefs combine. \citet{golub2010naive}
characterize when this process is wise. A large society converges to the truth if
and only if no single agent stays too influential as the society grows; prominent
agents are the main obstruction. A parallel Bayesian line reaches the same tension
through sequential observational learning \citep{bala1998learning,
acemoglu2011opinion}. This theory is elegant and its predictions are sharp. But it
rests on one assumption that fails for LLM agents: that an agent aggregates
everything its network exposes it to.

But an LLM agent does not aggregate this way. Its context is limited and its
attention is selective: it incorporates only part of what it is exposed to, and
concentrates on the most salient sources. What combines beliefs is
therefore not the visible network, but a reweighted version that concentrates
influence on a few sources (Figure~\ref{fig:framework}). Prior work places LLM
agents in networks and documents their failures: agents slide into conformity,
herding toward the most confident peer even when a minority is better supported
\citep{cho2025herd}, and grow increasingly sycophantic in repeated debate
\citep{pitre2025consensagent}. These studies propose fixes, but they do not
explain when a population herds and when it stays wise, nor do they expose a
control that moves it between the two regimes.

We close this gap with \textsc{SNLA} (Social Networks of LLM Agents),
a social-network model in which exposure and realized influence are
distinct. Starting from the visible network, \textsc{SNLA} derives a
\emph{realized-influence operator} determined by two quantities: the
\emph{social power} of each source \citep{bonacich1987power,
brin1998pagerank}, and an \emph{attention width} that controls how
sharply finite attention concentrates on the most powerful sources.
The theory is developed for a tractable proxy and validated against
the operational LLM system via a bridge theorem. Classical opinion
dynamics are recovered exactly when attention is unselective; the
regime of interest lies away from it.

Our contributions are summarized as follows:\\
$\triangleright$ \textbf{Realized influence as a derived object.} We introduce
\textsc{SNLA}, a framework in which influence is not read off the
interaction structure but \emph{derived} from it through social power
and a finite attention width, separating visible exposure from realized
influence (Section~\ref{sec:problem}).\\
$\triangleright$ \textbf{Attention width controls collective intelligence.} A
\emph{bridge theorem} couples the LLM population to a tractable proxy
path-wise (Theorem~\ref{thm:bridge}); a \emph{two-regime theorem} shows
attention width drives the herding--wisdom transition on the proxy, reaching
wisdom-of-crowds accuracy only when every agent hears comparably many
others (Theorem~\ref{thm:beta}); an \emph{equalization theorem} shows failure modes share a common cause, which a decentralized pricing
protocol removes, eliminating herding in the exact-clearing limit and
substantially reducing it online (Theorem~\ref{thm:eq}).\\
$\triangleright$ \textbf{Empirical validation.} A controlled testbed with scalar
signals and anchored updating ($\lambda_{\max}{<}1$) directly validates the
variance predictions and the proxy--LLM bridge; on operator-controlled variants
of three recognized benchmarks, the herding--wisdom transition, the placement
effect, and the equalizer's removal of herding all reproduce as qualitative
swings in collective accuracy (Section~\ref{sec:experiments}).

\section{Related Work}
\noindent\textbf{Social networks.}
Classical opinion dynamics reads influence directly off the network. In the
DeGroot \citep{degroot1974reaching} and Friedkin--Johnsen
\citep{friedkin1990social} models, each agent updates toward a weighted average
of its neighbors, and \citet{golub2010naive} characterize when this yields
accurate collective estimates. A parallel Bayesian line reaches similar
conclusions under sequential updating \citep{bala1998learning,
acemoglu2011opinion}. These models were designed for human agents who read
everything they are exposed to; for LLM agents, finite attention realizes only
part of the exposure as influence. \textsc{SNLA} builds on this tradition but
replaces the network weights with an influence operator derived from social
power and attention width, recovering the classical models as a special case.

\noindent\textbf{LLM multi-agent systems and collective behavior.} Prior work shows LLM
populations reproducing human-like opinion dynamics and network structure
\citep{chuang2024simulating, papachristou2025network}, with consensus degrading
at scale \citep{de2024language}, debate helping only when independent
\citep{du2024improving, liang2024encouraging, chan2024chateval,
taubenfeld2024systematic}, and single persuasive agents swaying whole groups
\citep{amayuelas2024multiagent}. A separate line wires such agents into
systems for software development \citep{qian2024chatdev}, role-play
\citep{li2023camel}, and task-solving \citep{wu2023autogen, hong2024metagpt,
guo2024large}. Neither asks what a topology actually realizes: the first
documents herding without a control that reverses it; the second treats it as
a design choice for performance, not a determinant of influence. \textsc{SNLA}
asks that question directly, separating exposure from realized influence to
derive when a population herds or stays wise.

\noindent\textbf{Mixture-of-experts.}
MoE routing exhibits the same column pathology: a
temperature-softmax over gate scores collapses tokens onto a few
experts (\emph{router collapse}), and the
standard fixes, load-balancing losses \citep{shazeer2017outrageously,
fedus2022switch}, Sinkhorn routing \citep{clark2022unified}, and
loss-free per-expert bias \citep{wang2024auxiliary}, are instances of
our equalizer, with the exposure price $y_j$ the analogue of the
per-expert load bias. \textsc{SNLA} adds the collective-estimation
account these lack: when concentration caps the effective sample size,
and when balancing restores it.
\section{Problem Formulation}
\label{sec:problem}

In this section, we formalize the social network induced by a
population of LLM agents. The central object is the
\emph{realized-influence operator}, which records how much each source
actually shapes another agent's beliefs, as distinct from the visible
exposure that the network makes available.

\subsection{Opinion Dynamics and Collective Estimation}
\label{sec:prelim-od}

\paragraph{Opinion dynamics.}
In the DeGroot model \citep{degroot1974reaching}, each agent updates its
belief as a weighted average of its neighbors'; the Friedkin--Johnsen
model \citep{friedkin1990social} adds anchoring:
\[
b_i(t+1) = \sum_j W_{ij}\,b_j(t) ,
\qquad
b_i(t+1) = (1 - \lambda_i)\,b_i^0 + \lambda_i \sum_j W_{ij}\,b_j(t) ,
\]
where $\mW \in \R^{n \times n}$ is row-stochastic, $\lambda_i \in [0,
1]$ is the \emph{anchoring weight}, and $b_i^0 \coloneqq b_i(0)$ the
initial belief. The influence matrix is exogenous: the network
alone determines how beliefs combine.

\paragraph{Collective estimation.}
Each agent begins with a \emph{private signal} $b_i^0 = \theta^\star +
\xi_i$, where $\xi_i$ is independent across agents with mean zero and
variance $s_i^2$. The \emph{collective estimate} at horizon $T$ is the
population average $\hat{\theta}_T = \frac{1}{n} \sum_i b_i(T) = \sum_i
q_{T,i}\, b_i^0$, a weighted combination of the initial signals with
weight vector $q_T \in \R^n$, $q_T^\top \vone = 1$. In the
homoskedastic case $s_i^2 \equiv s^2$, its variance is:
\[
\Var(\hat{\theta}_T) = s^2 \|q_T\|_2^2,
\]
minimized at $s^2/n$ when $q_T = \frac{1}{n}\vone$, the
wisdom-of-crowds optimum \citep{golub2010naive}. A weight concentrated
on a few agents instead keeps the variance bounded away from $s^2/n$
regardless of population size, the herding regime. Whether a
population is wise or herds is therefore set by $q_T$, which the
network produces as beliefs evolve.

\subsection{Social Networks of LLM Agents}
\label{sec:snla-def}

\paragraph{Agents and events.}
We model a population $\mathcal{A} = \{1, \ldots, n\}$ of $n$ LLM
agents. Each agent $i$ is a triple $(\rho_i, \mu_i, m_i)$, where
$\rho_i$ is a fixed \emph{persona} that anchors its belief, $\mu_i(t)$
a \emph{mutable state} holding its memory, and $m_i$ a \emph{backbone
model} mapping context to an emission. Each agent holds a declared
belief $b_i(t) \in \R$ about an unknown scalar $\theta^\star \in \R$;
we write $b(t) = (b_1(t), \ldots, b_n(t))^\top \in \R^n$ for the
stacked belief vector and $b^0 = b(0)$ for the initial signals.
At each round $t \in \{0, 1, 2, \ldots\}$, agents emit \emph{events};
with $\mathcal{X}$ the space of message texts, an event
\[
e_m = (s_m,\, R_m,\, x_m,\, b_m,\, t_m)
\;\in\;
\mathcal{A} \times 2^{\mathcal{A}} \times \mathcal{X} \times \R
\times \R_{\geq 0}
\]
records its sender, recipient set, message text, declared belief, and
timestamp. The history up to time $t$ is $\mathcal{E}_{\leq t} = \{e_m
: t_m \leq t\}$, from which the network is induced.

\paragraph{Exposure and realized-influence graphs.}
Classical opinion dynamics distinguishes the \emph{communication
graph}, recording which agents are aware of one another's opinions,
from the \emph{influence graph}, recording who actually shapes whom.
In the classical setting the two coincide, since an agent that reads a
neighbor also aggregates it \citep{proskurnikov2017tutorial}. Finite
attention breaks this coincidence, and \textsc{SNLA} is built on the
separation into two weighted directed graphs on the agents:
\[
\mathcal{G}_{\bar{\mW}}(t) =
\bigl(\mathcal{A},\, \mathcal{E}_{\bar{\mW}}(t),\, \bar{\mW}(t)\bigr),
\qquad
\mathcal{G}_{\mC}(t) =
\bigl(\mathcal{A},\, \mathcal{E}_{\mC}(t),\, \mC(t)\bigr),
\]
both with node set $\mathcal{A}$ and a row-stochastic weighted
adjacency matrix in $\R^{n \times n}$.

In the \emph{visible exposure graph} $\mathcal{G}_{\bar{\mW}}(t)$, an
arc $j \to i$ is present whenever agent $i$ has been exposed to a
message from agent $j$ in $\mathcal{E}_{\leq t}$, with weight
$\bar{W}_{ij}(t) \in [0, 1]$ the attention agent $i$ places on
source $j$. The \emph{realized-influence graph} $\mathcal{G}_{\mC}(t)$
records the influence that finite attention realizes: an agent can be
influenced only by sources it is exposed to, so its arcs are contained
in the exposure arcs,
\[
\mathcal{E}_{\mC}(t) \subseteq \mathcal{E}_{\bar{\mW}}(t),
\]
and $\mC_{ij}(t) \in [0, 1]$ is the weight with which agent $i$
aggregates source $j$'s belief. Together with a \emph{social power} vector $\pi(t) \in \Delta_n$ that weights each agent by its centrality as a source in
$\mathcal{G}_{\bar{\mW}}(t)$, these objects form the \emph{social
network of LLM agents} (\textsc{SNLA}), the time-indexed tuple:
\[
\mathcal{S}(t)
= \bigl(\mathcal{A},\, \mathcal{E}_{\leq t},\, \bar{\mW}(t),\,
\pi(t),\, \mC(t)\bigr).
\]
Its defining feature is that $\mC(t) \neq \bar{\mW}(t)$ in general; the
classical model is recovered only when attention is unselective and
social power uniform, at which $\mC(t) = \bar{\mW}(t)$.
\section{Proposed Methodology}
\label{sec:method}

In this section, we construct the realized-influence operator $\mC(t)$
that determines how much each source shapes another agent's
beliefs. We then give a decentralized protocol that keeps any single
source from dominating the population's attention, and state the
theoretical guarantees that follow.

\subsection{The Realized-Influence Operator}
\label{sec:method-cmi}

\paragraph{Exposure and social power.}
The operator is built in two steps from the event history: we first
turn the history into a visible exposure graph, then weight its sources
by their standing in that graph. Fix a memory kernel $\kappa :
\R_{\geq 0} \to \R_{\geq 0}$ that discounts old events and per-event
qualities $q_m \geq 0$. The unnormalized exposure of agent $i$ to
source $j$ accumulates every message $j$ sent that $i$ received:
\[
A_{ij}(t) = \sum_{e_m \in \mathcal{E}_{\leq t}}
\mathbf{1}\{s_m = j,\, i \in R_m\}\,\kappa(t - t_m)\, q_m,
\]
and row-normalizing gives the row-stochastic exposure matrix
$\bar{\mW}(t)$. Exposure alone treats all sources symmetrically, but an
agent carries more standing when influential agents attend to it. We
score this standing by the Katz--Bonacich centrality of the exposure
graph \citep{bonacich1987power}, for a damping factor $\zeta \in [0,
1)$:
\[
\pi(t)^\top
= \frac{1 - \zeta}{n}\,\vone^\top
\bigl(\mI - \zeta \bar{\mW}(t)\bigr)^{-1}
\in \Delta_n ,
\]
which weights each source by exposure accumulated over paths of every
length. At $\zeta = 0$ power is uniform and influence reduces to
exposure alone.

\paragraph{Realized influence.}
Exposure and power say how much each source \emph{could}
influence agent $i$; a finite context window forces $i$ to realize only
part of it. We form a per-pair \emph{score} $s_{ij}(t) \coloneqq
\bar{W}_{ij}(t)\,\pi_j(t)$ --- multiplicative, so a source matters to $i$
only when $i$ is exposed to it \emph{and} it carries standing --- and pass
each row through a temperature-$\beta$ softmax over the log-scores:
\begin{equation}
\label{eq:baseline-op}
s_{ij}(t) \coloneqq \bar{W}_{ij}(t)\,\pi_j(t) ,
\qquad
\mC_{ij}(t)
= \frac{s_{ij}(t)^{1/\beta}}
{\sum_k s_{ik}(t)^{1/\beta}} .
\end{equation}
The width $\beta$ is the temperature, controlling how sharply each
reader's budget concentrates: $\beta \to 0$ collapses the row onto its
highest-scoring source (hard attention), $\beta = 1$ allocates
proportionally, and $\beta \to \infty$ spreads it uniformly over
exposed sources. The result $\mC(t)$ is the realized-influence operator
of Section~\ref{sec:problem}: row-stochastic, supported on
$\bar{\mW}(t)$, and equal to $\bar{\mW}(t)$ at the classical point
$\zeta = 0, \beta = 1$. It drives the opinion dynamics in place of the
classical fixed matrix, updated each round as new events arrive. In the
LLM harness, each reader is shown the smallest set of peers covering
$0.9$ of its row mass (Appendix~\ref{app:beta}); the residual $\delta$
(Section~\ref{sec:method-theory}) is measured against the full $\mC$, so
this hard cutoff enters as a bounded contribution rather than an
unmodeled gap.

\subsection{Equalized Context Allocation}
\label{sec:method-eq}

\paragraph{Exposure pricing.} The operator just built has a structural weakness that the theory of
Section~\ref{sec:method-theory} traces to a single source. Row
normalization bounds each reader's budget, but nothing bounds the
source side: a high-power source can occupy a large share of the
population's contexts, so its column sum grows without bound,
\[
\sum_i \mC_{ij}(t) \ \gg\ 1 .
\]
We correct this with a decentralized pricing protocol: each source $j$
holds a scalar \emph{exposure price} $y_j > 0$ (initialized at one),
which discounts its weight in every reader's allocation and is updated
by the attention share it received,
\begin{equation}
\label{eq:equalizer-op}
\mC_{ij}(t)
= \frac{s_{ij}(t)^{1/\beta} / y_j}
{\sum_k s_{ik}(t)^{1/\beta} / y_k} ,
\qquad
y_j \leftarrow y_j \cdot a_j , \quad a_j = \sum_i \mC_{ij}(t) .
\end{equation}
At the fixed point $a_j = 1$ for all $j$, so $\mC(t)$ has unit row and
column sums, $\vone^\top \mC(t) = \vone^\top$ and $\mC(t)\vone = \vone$:
the realized influence is doubly stochastic. This is one Sinkhorn
iteration per round \citep{sinkhorn1967diagonal}, decentralized to $O(d_i)$
work and one broadcast scalar per agent, and (in its cleared limit)
invariant to the social-power weighting, attenuates exactly the
high-power sources that would otherwise dominate.
Algorithm~\ref{alg:snla-eq} collects the full per-round loop; the
Sinkhorn equivalence and a targeted heteroskedastic variant are in
Appendix~\ref{app:pricing}.

\begin{algorithm}[t]
\DontPrintSemicolon
\SetCommentSty{textnormal}
\SetKwInOut{Input}{Input}
\SetKwInOut{Output}{Output}
\Input{agents $\{(\rho_i, \mu_i, m_i)\}_{i=1}^n$, private signals
$b_i^0$, parameters $\beta, \zeta, T$}
\Output{collective estimate $\hat{\theta}_T$}
initialize beliefs $b_i(0) \leftarrow b_i^0$ and prices $y_j
\leftarrow 1$ for all $i, j$\;
\For{$t = 0, 1, \ldots, T-1$}{
  form exposure $\bar{\mW}(t)$ and social power $\pi(t)$ from history
  $\mathcal{E}_{\leq t}$\;
  \textbf{Reader rule:} priced temperature-$\beta$ allocation\;
  $\mC_{ij}(t) \leftarrow
  \dfrac{s_{ij}(t)^{1/\beta} / y_j}
  {\sum_k s_{ik}(t)^{1/\beta} / y_k}$
  \quad for all $i, j$\;
  \ForEach{agent $i$}{
    build context from persona $\rho_i$, state $\mu_i(t)$, and
    influence row $\mC_{i:}(t)$\;
    emit belief, message, recipients $(b_i(t{+}1), x_i(t), R_i(t))
    \sim m_i$\;
    append event $e_i(t)$ to $\mathcal{E}$ and update state
    $\mu_i(t{+}1)$\;
  }
  \textbf{Source rule:} each source raises its price by its
  attention share\;
  $a_j \leftarrow \sum_i \mC_{ij}(t)$, \quad $y_j \leftarrow y_j
  \cdot a_j$ \quad for all $j$\;
}
\Return{$\hat{\theta}_T = \frac{1}{n} \sum_i b_i(T)$}
\caption{\textsc{SNLA-EQ}: realized influence with equalized context
allocation.}
\label{alg:snla-eq}
\end{algorithm}

\subsection{Theoretical Guarantees}
\label{sec:method-theory}

With the allocator in place, we now analyze what
it implies for collective accuracy. The operational beliefs are emitted
by LLMs and are not directly analyzable, so we pair them with a
tractable proxy driven by the same realized influence, an anchored
best-response (Friedkin--Johnsen) recursion:
\begin{equation}
\label{eq:proxy-body}
b_i^{\mathrm{BR}}(t+1)
= (1 - \lambda_i)\, b_i^0
+ \lambda_i \sum_j \mC_{ij}(t)\, b_j^{\mathrm{BR}}(t),
\qquad \lambda_{\max} = \max_i \lambda_i ,
\end{equation}
with anchoring weights $\lambda_i \in [0, 1)$. The operational emission
need not solve this recursion exactly; we measure its per-step deviation
by the residual of agent $i$ at round $t$, evaluated at the operational
system's own beliefs $b(t)$:
\begin{equation}
\label{eq:residual}
\hat\delta_i(t)
= \Bigl| b_i(t+1) - (1 - \lambda_i)\, b_i^0
- \lambda_i \sum_j \mC_{ij}(t)\, b_j(t) \Bigr| ,
\end{equation}
which is computable from the event log once the run is complete. Unrolling
the proxy recursion writes its estimate as a fixed weighting of the
initial signals, $\hat{\theta}_T^{\mathrm{BR}} = q_T^\top b^0$, where the
collective weight $q_T$ is determined by $\{\mC(t)\}$
(Appendix~\ref{app:proofs}). The wisdom-versus-herding question of
Section~\ref{sec:prelim-od} thus becomes a question about $q_T$, which
the three results below answer: the first connects the proxy to the
operational system, and the next two characterize the consensus weight
under the baseline and the equalized allocator. All proofs, constants,
and assumptions are in Appendix~\ref{app:proofs}.

\begin{restatable}[Bridge]{theorem}{bridgethm}
\label{thm:bridge}
Suppose every LLM emission is a $\delta$-approximate best response
evaluated at its own beliefs, i.e.\ $\hat\delta_i(t) \leq \delta_T$ for
all $i$ and $t < T$, where $\delta_T = \max_{i,\, 0 \leq t < T}
\hat\delta_i(t)$ is the uniform residual~\eqref{eq:residual} and
$\lambda_{\max} < 1$. Then for every $t \leq T$ and every row-stochastic
sequence $\{\mC(t)\}$,
\begin{equation}
\label{eq:bridge-body}
\bigl\| b(t) - b^{\mathrm{BR}}(t) \bigr\|_\infty
\leq \frac{\delta_T}{1 - \lambda_{\max}} .
\end{equation}
\end{restatable}

The bound is path-wise and attained, and anchoring cannot be dropped: at $\lambda_{\max} = 1$ a
row-stochastic sequence and emission rule exists whose gap grows linearly in $t$
(Appendix~\ref{app:proofs}), a counterexample rather than a claim that every $\lambda_{\max}=1$
system diverges. The bound is non-vacuous in practice: an
explicit anchoring prompt keeps $\lambda$ comfortably below one on the anchored-persona cells
(median $0.25$; full breakdown in Section~\ref{sec:exp-benchmarks}), and
Figure~\ref{fig:bridge} confirms the bound holds on real agents.
The operational estimate inherits the proxy's accuracy, $|\hat{\theta}_T -
\hat{\theta}_T^{\mathrm{BR}}| \leq \delta_T/(1 - \lambda_{\max})$, and likewise its standard
deviation, so it suffices to analyze $q_T$, measured via the effective sample size
$N_{\mathrm{eff}}(q_T) = 1/\|q_T\|_2^2$ \citep{kish1965survey}, which gives variance
$s^2/N_{\mathrm{eff}}(q_T)$.

The next theorem fixes the exposure graph and varies the attention
width. It characterizes the consensus weight $\nu_{\mC(\beta)}$, the
stationary distribution of $\mC(\beta)$, which is the long-horizon limit
of $q_T$ under uniform anchoring as $\lambda \uparrow 1$
(Appendix~\ref{app:proofs}); accordingly $N_{\mathrm{eff}}(\beta) =
1/\|\nu_{\mC(\beta)}\|_2^2$. We reserve $\pi(t)$ for the social power
that builds $\mC$ and write $\nu$ for the consensus weight it induces;
the two are distinct.

\begin{restatable}[Attention width controls collective accuracy]{theorem}{betathm}
\label{thm:beta}
Let $\mC(\beta)$ be the fixed-graph realized influence with $\zeta \in
[0,1)$ and irreducible exposure support $B$, where $B_{ij} =
\mathbf{1}\{\bar{W}_{ij} > 0\}$. Write the degrees $d_i = \sum_j
B_{ij}$, $D = \diag(d_1, \ldots, d_n)$, $d_{\max} = \max_i d_i$, and let
$L$ be the log-spread and $\underline{g}$ the dominance gap of the scores
$s_{ij}$ (Appendix~\ref{app:proofs}). In the homoskedastic case:
\begin{enumerate}
\item \emph{(Wide, wisdom of crowds.)} If $B = B^\top$ is irreducible,
then $\mC(\beta) \to D^{-1} B$ and $\nu_{\mC(\beta)} \to d /
(\vone^\top d)$ as $\beta \to \infty$, so:
\[
\lim_{\beta \to \infty} N_{\mathrm{eff}}(\beta)
= \frac{(\sum_i d_i)^2}{\sum_i d_i^2} ,
\]
which equals $n$ if and only if $B$ is degree-regular; for finite
$\beta$ the consensus weight obeys $\bigl| \|\nu_{\mC(\beta)}\|_2^2 -
\|\nu^\star\|_2^2 \bigr| \leq 2\kappa_B (e^{L/\beta} - 1)$ with
$\nu^\star = d/(\vone^\top d)$ and $\kappa_B$ a condition number of
$D^{-1} B$ (Appendix~\ref{app:proofs}), so:$N_{\mathrm{eff}}(\beta) \to (\sum_i d_i)^2/\sum_i d_i^2$ by continuity.

\item \emph{(Narrow, herding.)} Assume $B$ is irreducible and the scores
admit a unique dominant pair: $j_i^\star = \argmax_k s_{ik}$ is unique
for every row, with $j_i^\star = h$ for all $i \neq h$ and $j_h^\star = g
\neq h$. Then with $\rho(\beta) = (d_{\max} - 1)\, e^{-\underline
g/\beta}$, whenever $\rho(\beta) < 1$,
\[
N_{\mathrm{eff}}(\beta) \leq \frac{2}{(1 - \rho(\beta))^2} ,
\]
so $N_{\mathrm{eff}}(\beta) \leq 8$ whenever $\rho(\beta) \leq \tfrac12$,
i.e.\ $\beta \leq \underline g / \log(2(d_{\max} - 1))$, independently of
$n$.
\end{enumerate}
\end{restatable}

Narrow attention concentrates each row of $\mC$ on the dominant pair,
collapsing the collective weight onto a few agents; wide attention
flattens it but reaches the optimum only under degree-regularity, the
sole remaining obstruction. Both failures are column imbalances of
$\mC$, which is exactly what the equalized allocator removes.

\begin{restatable}[Exact optimality under equalization]{theorem}{eqthm}
\label{thm:eq}
Under uniform anchoring $\lambda \in [0, 1)$ and signal exogeneity, let
$\bar{\delta}_{\mathrm{col}}$ be the largest column defect over
the run. Then the proxy weight satisfies:
\[
\Bigl\| q_T - \tfrac{1}{n}\vone \Bigr\|_1
\leq \bar{\delta}_{\mathrm{col}}
\Bigl( \tfrac{\lambda}{1-\lambda} + T\lambda^T \Bigr) ,
\]
for every horizon $T$ and every realized sequence. In particular, if
every $\mC(t)$ is doubly stochastic ($\bar{\delta}_{\mathrm{col}} = 0$),
then $q_T = \frac{1}{n}\vone$ exactly, so
$\Var(\hat{\theta}_T^{\mathrm{BR}}) = s^2/n$ and $N_{\mathrm{eff}} = n$.
The operational estimate inherits this up to the bridge gap,
\[
\Bigl| \mathrm{sd}(\hat{\theta}_T) - \tfrac{s}{\sqrt{n}} \Bigr|
\leq \frac{\delta_T}{1-\lambda}
+ s\, \bar{\delta}_{\mathrm{col}}
\Bigl( \tfrac{\lambda}{1-\lambda} + T\lambda^T \Bigr) .
\]
\end{restatable}

Each price round is one Sinkhorn iteration, which need not clear the
market exactly in a single pass; the guarantee is stated in the measured
column defect $\bar{\delta}_{\mathrm{col}}$ actually achieved at run
time, and is exact only in the doubly stochastic limit, reached by
iterating the price update to tolerance
(Appendix~\ref{app:pricing}). Unlike the baseline, which collapses to
$N_{\mathrm{eff}} \leq 8$ under a dominant pair no matter how large $n$
is, the equalized allocator drives $q_T$ to uniform at a rate set by
$\bar{\delta}_{\mathrm{col}}$, with no mixing, aperiodicity, or
connectivity hypothesis on the realized sequence.
\section{Experiments}
\label{sec:experiments}

\subsection{Experimental setup}
\label{sec:exp-setup}

\paragraph{Benchmarks.} \emph{HiddenBench}~\citep{li2025hiddenbench} poses $12$ hidden-profile
reasoning problems, where each of $n{=}24$ agents holds a different clue that the group must pool
to answer correctly; \emph{Werewolf}~\citep{xu2023language} is a $16$-agent social deduction game
where villagers deliberate and vote to unmask a hidden adversarial werewolf; and
\emph{AgentsNet}~\citep{grotschla2025agentsnet} places $n\in\{8,16,24,32,50\}$ agents as nodes in a
graph, where each communicates only with its neighbours to reach a conflict-free colouring. We
additionally map the theorem's boundary on full-information \emph{Debate}~\citep{du2024improving}
and commons-dilemma \emph{GovSim}~\citep{piatti2024cooperate}. Full environment specifications are
in Appendix~\ref{app:envs}.

\textbf{Models.} Each agent is Qwen2.5-7B-Instruct~\citep{qwen2025qwen25technicalreport} by
default, run offline with greedy decoding (temperature-$0$); each condition is repeated over $5$
seeds re-drawing tasks and position assignments, and the shaded bands in
Figures~\ref{fig:three} and~\ref{fig:online} span their min--max. A network is one model
instantiated as $n$ agents that \emph{share weights}, differing only in persona, private
information, and position, so any collective effect is structural rather than a capability
artifact --- the injected confident-wrong leader is the same base model, so herding traces to its
position, not lower capability. Two experiments break this homogeneity: the AgentsNet society
pairs strong 7B hubs with weak 1.5B agents, and the capability ladder scales the central agent
across Qwen2.5-$\{1.5,3,7,14\}$B, replicated on Llama-3-$\{1,3,8\}$B~\citep{grattafiori2024llama}.
Full model, operator, seed, prompt, and hardware details are in
Appendices~\ref{app:models}--\ref{app:hardware}.

\subsection{Main results}
\label{sec:exp-benchmarks}

\paragraph{Context width $\beta$ controls herding versus wisdom.}
A single knob, the context width $\beta$, swings the collective between herding and
wisdom on both aggregation-dependent tasks (Figure~\ref{fig:three}A,B).
In \emph{HiddenBench}, the baseline collapses onto the confident-wrong hub at narrow
$\beta$ and recovers at wide $\beta$, improving collective accuracy by
$\mathbf{+0.64}$. A single item traced through both regimes, with the exact prompts and the
model's real replies, is in Appendix~\ref{app:prompts}. The equalizer removes this collapse, while
the control shows it disappears without the dominant-wrong source.
\emph{Werewolf} shows the same pattern with the game's intrinsic wrong source seated at the
influence hub: the
baseline moves from near-zero accuracy at narrow $\beta$ to substantially higher
accuracy at wide $\beta$ ($\mathbf{+0.61}$), while the equalizer and control stay
above the herding floor. The same collapse and equalizer recovery replicate on
Llama-3.1-8B (Table~\ref{tab:app-llama}). The individual and oracle lines mark the lower and upper
reference levels; the oracle gap is diagnosed in Appendix~\ref{app:bridge}. The full
per-$\beta$ breakdown, including the causal adversary-on/off decomposition, is in
Appendices~\ref{app:hb} (HiddenBench) and~\ref{app:ww} (Werewolf).

\begin{figure}[h]\centering
  \includegraphics[width=\linewidth]{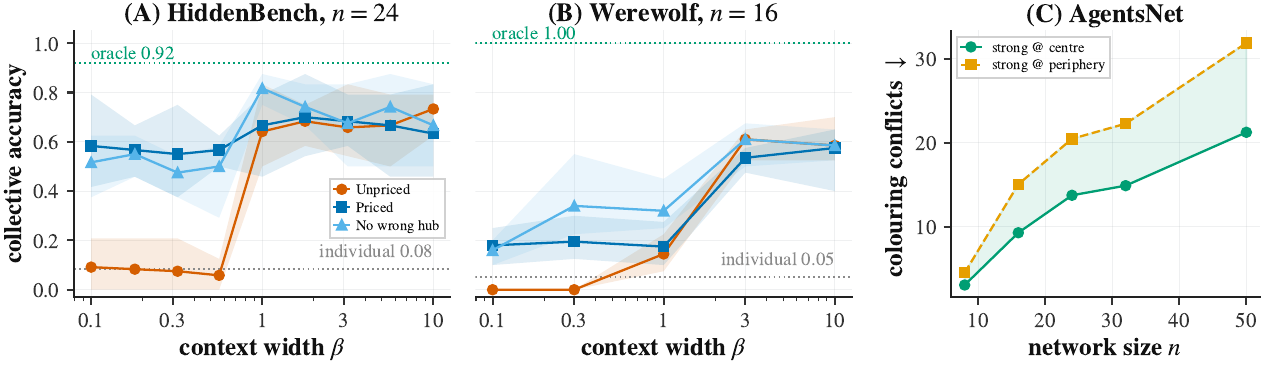}
  \caption{\textbf{One influence operator, three recognized environments, two facets.}
  \textbf{(A,~B)} Collective accuracy vs.\ context width $\beta$ on \emph{HiddenBench}
  ($n{=}24$) and \emph{Werewolf} ($n{=}16$). Orange is the \emph{baseline} allocator,
  blue is the \emph{equalizer}, and light blue is the \emph{control} that removes the
  dominant-wrong source. Dotted lines show individual and oracle reference levels;
  shaded bands in (A,~B) span the seed-to-seed min--max over $5$ seeds.
  \textbf{(C)} Colouring conflicts vs.\ network size $n$ on \emph{AgentsNet}, comparing
  strong agents placed at the centre (green) with strong agents placed at the periphery (gold).}
  \label{fig:three}
\end{figure}

\paragraph{Where capability sits controls coordination.}
On a coordination task the operator exposes a different facet
(Figure~\ref{fig:three}C). \emph{AgentsNet} has agents self-organise into complementary
colours. A heterogeneous $7$B$\times$$1.5$B society fixes the capability budget,
varying only where the strong agents sit: strong agents at the centre yield fewer conflicts
than at the periphery, and the advantage grows with scale,
from $\mathbf{+1.5}$ at $n{=}8$ to $\mathbf{+10.6}$ at $n{=}50$. Coordination has no
single wrong answer to herd onto, isolating placement rather than the
herding--wisdom transition. The placement ablation matrix (social-power and
Sinkhorn invariance, the weak-hub Watts--Strogatz control) is in Appendix~\ref{app:agentsnet}.
\begin{figure}[t]\centering
  \includegraphics[width=1\linewidth]{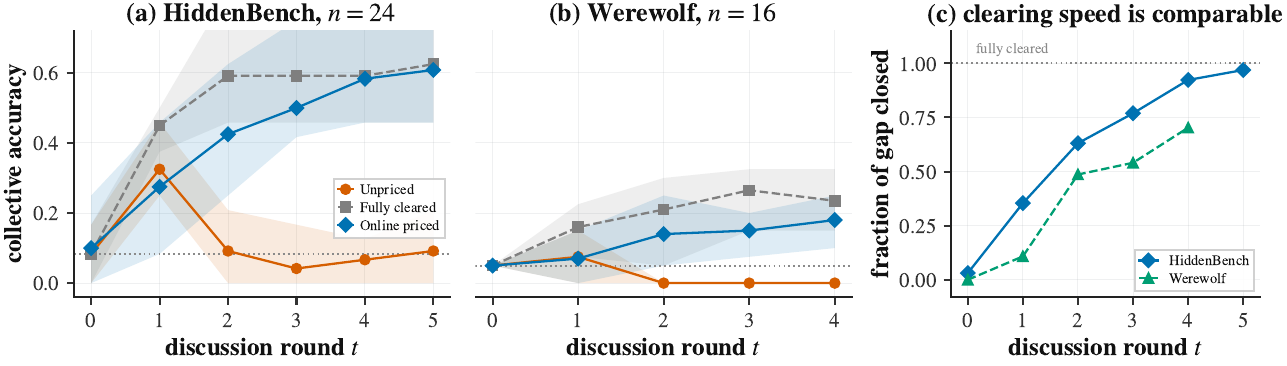}
  \caption{\textbf{The equalizer clears herding online, round by round.}
  Per-round collective accuracy at narrow $\beta$ on \emph{HiddenBench} ($n{=}24$)
  and \emph{Werewolf} ($n{=}16$). Orange is the \emph{baseline}; grey dashed is the
  \emph{operator}, obtained by clearing the equalized operator to its fixed point each
  round; blue is \emph{online pricing}, which applies one price update per discussion
  round. Shaded bands span the min--max over $5$ seeds. \textbf{(c)} Fraction of the
  narrow-$\beta$ gap closed over rounds on HiddenBench and Werewolf.}
  \label{fig:online}
\end{figure}

\paragraph{Online pricing clears herding round by round.}
Figure~\ref{fig:online} separates the ideal equalized operator from its online
implementation, which runs one price update per discussion round rather than to
convergence (the decentralized protocol of Section~\ref{sec:method-eq}). Because
prices start uniform, the first step sits at the baseline and the collective follows
the confident-wrong source; subsequent updates reduce the captured column, so the
online curve rises toward the operator curve. The recovery is nearly complete on
\emph{HiddenBench} (Figure~\ref{fig:online}a, $0.10\!\to\!0.61$) but partial on the
shorter-horizon \emph{Werewolf} (Figure~\ref{fig:online}b, $0.05\!\to\!0.18$ against
the operator's $0.24$). The two close their gaps at a similar per-round pace, so the
remaining Werewolf gap is a horizon effect, not a different clearing mechanism.

\begin{figure}[h]\centering
  \includegraphics[width=\linewidth]{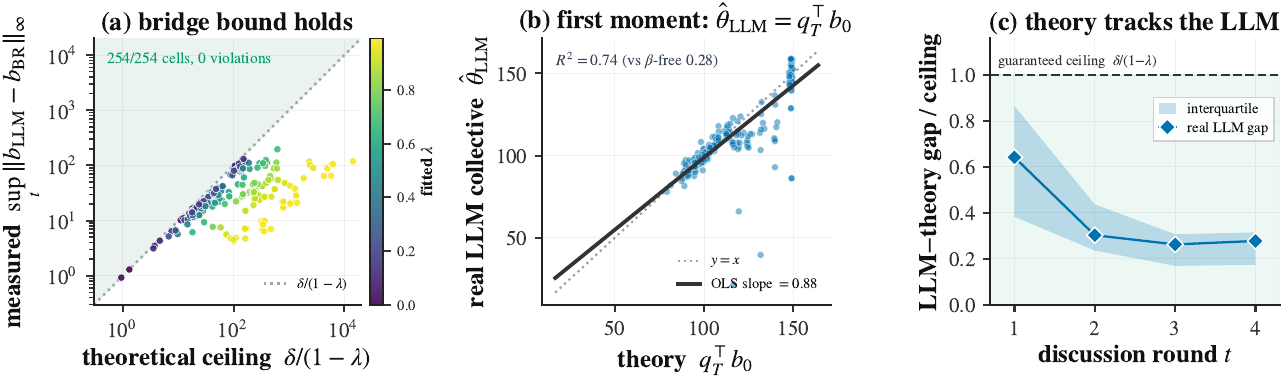}
  \caption{\textbf{The proxy--LLM bridge (Theorem~\ref{thm:bridge}) on real Qwen data.}
  \textbf{(a)} Measured sup-norm discrepancy $\sup_t\lVert b_{\mathrm{LLM}}-b_{\mathrm{BR}}\rVert_\infty$
  vs.\ the theoretical ceiling $\delta_T/(1-\lambda_{\max})$, pooling $254$ cells --- one per
  (arm, $\beta$, seed) with $\lambda_{\max}<1$ --- across four $8$-seed testbed runs spanning
  Qwen2.5-3B and -7B backbones, including one 7B run with an explicit anchoring-persona prompt
  and one 3B run at reduced damping ($\zeta{=}0.6$); coloured by fitted $\lambda_{\max}$; dotted
  line is $y=x$.
  \textbf{(b)} Real LLM collective estimate $\hat\theta_{\mathrm{LLM}}$ vs.\ the proxy's
  first-moment prediction $q_T^{\!\top}b_0$, pooling the $8$-seed 7B run from panel (a) with a
  further $24$-seed 7B sweep; dotted is $y=x$, solid is the OLS fit.
  \textbf{(c)} LLM--theory gap, normalized by the guaranteed ceiling $\delta/(1-\lambda)$ (dashed),
  across $T{=}4$ discussion rounds on the $8$-seed anchored-persona 7B run from panel (a) alone;
  band is the interquartile range.}
  \label{fig:bridge}
\end{figure}

\paragraph{The proxy predicts real LLM estimates.}
The discussion benchmarks run DeGroot updating ($\lambda{=}1$), the unanchored limit
of Theorem~\ref{thm:beta}: the proxy's effective sample size collapses and recovers
with $\beta$, and the LLM's collective accuracy tracks that swing (Figure~\ref{fig:three}A,B).
At $\lambda{=}1$ this is empirical, not a consequence of Theorem~\ref{thm:bridge}, whose
bound is vacuous here (Proposition~\ref{prop:sharp}); the two share the $\beta$-gating
mechanism, not a theorem mapping $N_{\mathrm{eff}}$ onto accuracy. The bridge bound requires
anchoring, so we certify the step-by-step correspondence on the anchored testbed: there the
LLM estimate is the influence-weighted average the theory predicts (slope $0.88$, $R^2{=}0.74$),
and the sup-norm bound holds on all $254$ cells (Figure~\ref{fig:bridge}). The anchoring
$\lambda$ is fit per agent, so the ceiling is fitted not prescribed---tight on the anchored
subset, satisfied over the pool; full provenance is in Appendix~\ref{app:bridge}.

\subsection{Ablations}
\label{sec:exp-ablations}
The herding floor is an effective-sample-size effect, not a small-network artifact: it holds at
every society size on the real benchmarks (Werewolf for $n\in\{6,8,10,12,16\}$, HiddenBench for
$n\in\{8,16,24\}$; full breakdown in Appendix~\ref{app:hb} and~\ref{app:ww}), so a larger network
does not dilute a captured hub. Table~\ref{tab:ablations} sweeps the remaining controls in the
controlled testbed (the ablation grid, plus the model-scale sweep, is in
Appendix~\ref{app:ablations} and~\ref{app:testbed}): the two-regime is invariant to
follower self-reliance and survives five
structural variants of the interaction model. It is, however, gated causally by social power,
vanishing entirely when the gate is off ($\zeta{=}0$).
\begin{table}[t]\centering\small
\caption{\textbf{Ablations (controlled testbed, Qwen2.5-3B, $8$ seeds).} Collective error, lower is
  better; narrow $\beta{=}0.1$ is the herding regime, wide $\beta{=}10$ the aggregate.
  Variants: memory, exposure-pricing, late-joiner, distance-decay, skeptic.}
  \label{tab:ablations}
  \begin{tabular}{llccl}\toprule
  Ablation & range & narrow $\beta$ & wide $\beta$ & reading \\\midrule
  social power $\zeta{=}0$  & gate off      & $\mathbf{5.5}$ & $4.9$  & \textbf{no damping $\Rightarrow$ no herding} \\
  social power $\zeta{>}0$  & $0.2$--$0.95$ & $\approx 42$    & $4$--$6$  & damping on $\Rightarrow$ herding \\
  self-reliance $\lambda$   & $0.3$--$0.7$  & $42$           & $11.5$ & robust to follower self-weight \\
  equalizer $k$ (Sinkhorn)  & $1$--$4$      & \multicolumn{2}{c}{column defect $0.14\!\to\!0.037$ (Fig.~\ref{fig:app-abl}a)} & one knob clears the market \\
  model variants            & $5$ kinds     & $18$--$59$      & $2$--$15$ & herding persists; eq clears each \\
  \bottomrule\end{tabular}
\end{table}
\section{Conclusion}
\label{sec:conclusion}
SNLA shows that a single parameter controls collective intelligence: the
attention width $\beta$. It governs whether an LLM-agent population pools its
information or herds. Narrow $\beta$ collapses collective accuracy onto a few
dominant sources; wide $\beta$ recovers wisdom-of-crowds behavior, once exposure
is balanced. One failure then remains, a high-power source that occupies too
many contexts, and a decentralized pricing rule removes it, restoring the
optimal collective weights whenever a dominant source is present, with no
retraining.

\bibliographystyle{plainnat}
\bibliography{Reference}

\clearpage
\crefalias{section}{appendix}
\crefalias{subsection}{appendix}
\vspace{1cm}
\begin{center}
    {\LARGE \textbf{Supplementary Materials}}
\end{center}

\section{Assumptions and Measurement of \texorpdfstring{$\delta$}{delta}}
\label{app:assumptions}

This appendix collects the assumptions of the three theorems and the
procedure measuring $\delta$ from the event log. Throughout, $P =
\mathrm{diag}(\lambda_1, \ldots, \lambda_n)$, $\lambda_{\max} = \max_i
\lambda_i$, and $\{\mC(t)\}$ is the row-stochastic realized-influence
sequence; operational and proxy beliefs $b(t), b^{\mathrm{BR}}(t)$ are
both initialized at $b^0$.

\begin{assumption}[Anchored cooperative regime]
\label{ass:fj}
The proxy is the quadratic, cooperative, anchored best response: the
anchoring weights satisfy $\lambda_i \in [0, 1)$ for all $i$, so that
$\lambda_{\max} < 1$.
\end{assumption}

The weight $\lambda_i < 1$ is the regime in which the persona anchor
$\rho_i$ exerts nonzero pull toward $b_i^0$. The restriction
$\lambda_{\max} < 1$ cannot in general be dropped:
Proposition~\ref{prop:sharp} exhibits a row-stochastic sequence and a
$\delta$-rational rule with $\lambda_{\max} = 1$ whose gap grows
linearly in $t$. This is a constructive counterexample, not a claim that
every $\lambda_{\max} = 1$ system diverges.

\begin{assumption}[$\delta$-rational emission]
\label{ass:delta}
There exists $\delta \geq 0$ such that, along the operational
trajectory,
\begin{equation}
\label{eq:delta}
\Bigl|
b_i(t+1)
- (1 - \lambda_i)\, b_i^0
- \lambda_i \sum_j \mC_{ij}(t)\, b_j(t)
\Bigr|
\ \leq\ \delta
\qquad
\text{for all } i \text{ and all } t \geq 0 ,
\end{equation}
where $\mC(t)$ is the realized-influence matrix produced by the
operational system at round $t$.
\end{assumption}

Two features of \eqref{eq:delta} matter. The best response is evaluated
at the LLM's \emph{own} beliefs $b(t)$, not the proxy
$b^{\mathrm{BR}}(t)$: a one-step condition presupposing nothing about
trajectory distance. And every quantity in \eqref{eq:delta} is logged,
so $\delta$ is measurable per emission, the assumption falsifiable.

\begin{assumption}[Signal model and signal--exposure independence]
\label{ass:indep}
The private signals are $b_i^0 = \theta^\star + \xi_i$ with
$\xi_1, \ldots, \xi_n$ independent, $\mathbb{E}[\xi_i] = 0$, and
$\mathrm{Var}(\xi_i) = s_i^2$, and the noise vector
$\xi = (\xi_1, \ldots, \xi_n)$ is independent of the realized sequence
$\{\mC(t)\}_{t \geq 0}$. The homoskedastic case is $s_i^2 \equiv s^2$.
\end{assumption}

Independence of $\xi$ from $\{\mC(t)\}$ holds when recipient selection
and event qualities depend on messages only through timing and routing,
not belief values; it is the exogenous-topology analogue from classical
opinion dynamics and matches our fixed-topology environments. It is what
licenses treating $q_T$ as fixed under $\mathbb{E}_\xi$ in
Proposition~\ref{prop:variance}: belief-dependent selection couples $q_T$
and $\xi$, voiding the conditional variance identities, and so lies
outside all variance-level results below. The Bridge Theorem, which does
not use Assumption~\ref{ass:indep}, is unaffected.

\paragraph{Measuring $\delta$.}
\label{app:measuring-delta}
Assumption~\ref{ass:delta} is checkable round by round. At round $t$,
both the realized-influence row $\mC_{i:}(t)$ and the operational
beliefs $b(t)$ are functions of the event log, so the per-emission
residual
\begin{equation}
\label{eq:delta-hat}
\hat\delta_i(t)
=
\Bigl|
b_i(t+1)
- (1 - \lambda_i)\, b_i^0
- \lambda_i \textstyle\sum_j \mC_{ij}(t)\, b_j(t)
\Bigr|
\end{equation}
is computable offline once the run is complete. The anchoring weights
$(\lambda_i)$ are not assumed known: they are fit per agent by least
squares over the run, minimizing $\sum_t \hat\delta_i(t)^2$ over
$\lambda_i \in [0, 1)$. For fixed $i$ this objective is a univariate
quadratic in $\lambda_i$, since $\hat\delta_i(t)$ is affine in
$\lambda_i$ at each $t$, so the fit is closed-form. The theorems require
the \emph{uniform} residual of Assumption~\ref{ass:delta}, so the value
that instantiates them is the maximum
\[
\delta_T = \max_{i,\, 0 \leq t < T}\, \hat\delta_i(t) ,
\]
not a quantile: an upper quantile bounds all but a fraction of the
emissions and yields only a trimmed, high-probability diagnostic, under
which the bridge bound holds after discarding the residual outliers. A
finite run of length $T$ measures $\delta_T$ and so supports the bridge
bound up to horizon $T$; extending it to all $t$ requires the residual
to remain uniformly bounded beyond the run, which is an assumption, not
a measurement. The empirical distribution of $\hat\delta_i(t)$ is itself
a finding: it quantifies how far a given backbone, persona, and topic
depart from a rational anchored opinion updater, and its tail governs
the gap between the uniform bound and its trimmed counterpart.

\section{Proof of \texorpdfstring{Theorem~\ref{thm:bridge}}{Theorem 1} (Bridge)}
\label{app:bridge-sec}
\label{app:proofs}

This appendix proves the three theorems and supporting results. The
development is self-contained: all constants are explicit, and the
stationary perturbation of the two-regime theorem is carried out from a
deviation-matrix identity (Lemma~\ref{lem:deviation}).

We use the following notation throughout. For a matrix $R$,
$\|R\|_\infty = \max_i \sum_j |R_{ij}|$ is the maximum absolute row sum,
which is the operator norm induced by the vector $\ell_\infty$ norm;
for a vector $v$, $\|v\|_p$ is the usual $\ell_p$ norm. A
row-stochastic matrix $\mC$ satisfies $\mC \geq 0$ entrywise and
$\mC \vone = \vone$. A nonnegative square matrix is \emph{irreducible}
if its support digraph is strongly connected. We write $e_j$ for the
$j$-th standard basis vector and $\Delta_n$ for the probability simplex
in $\R^n$. The exposure graph, social power, and realized influence are
$\bar{\mW}(t)$, $\pi(t)$, and $\mC(t)$ as defined in
Section~\ref{sec:method-cmi}; we abbreviate the per-pair
exposure--power score by
\begin{equation}
\label{eq:s-shorthand}
s_{ij}(t) \coloneqq \bar W_{ij}(t)\, \pi_j(t) ,
\end{equation}
so that the realized influence is the row-wise tempered allocation
$\mC_{ij}(t) = s_{ij}(t)^{1/\beta} / \sum_k s_{ik}(t)^{1/\beta}$. Since
$\pi_j(t) \geq \frac{1-\zeta}{n} > 0$ for every $j$
(Lemma~\ref{lem:power} below), $s_{ij}(t) > 0$ if and only if
$\bar W_{ij}(t) > 0$, so $\mC(t)$ has exactly the support of the
exposure graph.

\subsection{The proxy and its closed form}
\label{app:proxy}

The operational beliefs are emitted by the LLM backbones and are not
directly analyzable. The proxy replaces each emission by an anchored
best response driven by the \emph{same} realized influence sequence
$\{\mC(t)\}$. Writing the cooperative quadratic best response of
agent $i$ explicitly, the proxy is the recursion
\begin{equation}
\label{eq:proxy}
b^{\mathrm{BR}}(0) = b^0 ,
\qquad
b_i^{\mathrm{BR}}(t+1)
= (1 - \lambda_i)\, b_i^0
+ \lambda_i \sum_j \mC_{ij}(t)\, b_j^{\mathrm{BR}}(t) ,
\end{equation}
which in matrix form, with $P = \mathrm{diag}(\lambda_1, \ldots,
\lambda_n)$, reads
\begin{equation}
\label{eq:proxy-matrix}
b^{\mathrm{BR}}(t+1)
= (\mI - P)\, b^0 + P\, \mC(t)\, b^{\mathrm{BR}}(t) .
\end{equation}
The peer beliefs entering \eqref{eq:proxy} are the proxy's own
$b_j^{\mathrm{BR}}(t)$, never the operational $b_j(t)$: the proxy is a
self-contained linear dynamical system, coupled to the operational run
only through the realized matrices $\{\mC(t)\}$. Two classical models
are strict special cases. With uniform anchoring $\lambda_i \equiv
\lambda$ the recursion is the Friedkin--Johnsen model
\citep{friedkin1990social} on the realized influence; with $\lambda_i
\equiv 1$ it is the DeGroot model \citep{degroot1974reaching}.

\begin{lemma}[Best-response closed form]
\label{lem:closed}
Fix agent $i$ and round $t$, and let the peer beliefs
$b^{\mathrm{BR}}(t)$ be given. The cooperative quadratic best response
\[
b_i^{\mathrm{BR}}(t+1)
= \argmin_{b \in \R}\
(1 - \lambda_i)\, (b - b_i^0)^2
+ \lambda_i \sum_j \mC_{ij}(t)\, \bigl(b - b_j^{\mathrm{BR}}(t)\bigr)^2
\]
has the unique solution \eqref{eq:proxy}, for every $\lambda_i \in
[0,1]$ and every row-stochastic $\mC(t)$.
\end{lemma}

\begin{proof}
Write the objective as
\[
F(b)
= (1 - \lambda_i)\, (b - b_i^0)^2
+ \lambda_i \sum_j \mC_{ij}(t)\, \bigl(b - b_j^{\mathrm{BR}}(t)\bigr)^2 .
\]
The coefficient of $b^2$ is $(1-\lambda_i) + \lambda_i \sum_j
\mC_{ij}(t) = 1 > 0$, so $F$ is strictly convex with unique minimizer
at $F'(b) = 0$. Since
\[
F'(b)
= 2(1 - \lambda_i)(b - b_i^0)
+ 2 \lambda_i \sum_j \mC_{ij}(t)\,\bigl(b - b_j^{\mathrm{BR}}(t)\bigr)
= 2 \Bigl[ b - (1 - \lambda_i)\, b_i^0
- \lambda_i \sum_j \mC_{ij}(t)\, b_j^{\mathrm{BR}}(t) \Bigr]
\]
(using $\sum_j \mC_{ij}(t) = 1$), $F'(b) = 0$ gives \eqref{eq:proxy}.
\end{proof}

\subsection{Well-posedness of social power}
\label{app:power}

The realized influence is built from the exposure graph $\bar{\mW}(t)$
and the social power $\pi(t)$, the Katz--Bonacich centrality
\[
\pi(t)^\top
= \frac{1 - \zeta}{n}\, \vone^\top
\bigl(\mI - \zeta \bar{\mW}(t)\bigr)^{-1} ,
\qquad \zeta \in [0, 1) .
\]
The following lemma records that this is well defined and strictly
positive, a fact used repeatedly below (it guarantees that $\mC(t)$ has
exactly the support of $\bar{\mW}(t)$).

\begin{lemma}[Well-posedness of social power]
\label{lem:power}
For every row-stochastic $\bar{\mW}(t)$ and every $\zeta \in [0, 1)$,
the inverse $(\mI - \zeta \bar{\mW}(t))^{-1}$ exists, $\pi(t) \in
\Delta_n$, and
\[
\pi_j(t) \ \geq\ \frac{1 - \zeta}{n} \ >\ 0
\qquad \text{for every } j .
\]
At $\zeta = 0$, $\pi_j(t) = 1/n$ for all $j$.
\end{lemma}

\begin{proof}
Since $\bar{\mW}(t)$ is row-stochastic, $\|\bar{\mW}(t)\|_\infty = 1$,
hence $\|\zeta \bar{\mW}(t)\|_\infty = \zeta < 1$. The Neumann series
\[
\bigl(\mI - \zeta \bar{\mW}(t)\bigr)^{-1}
= \sum_{k=0}^\infty \zeta^k \bar{\mW}(t)^k
\]
converges absolutely, so the inverse exists. Every term is entrywise
nonnegative, so $\pi(t) \geq 0$; the $k = 0$ term gives
\[
\pi_j(t)
= \frac{1-\zeta}{n} \sum_{k \geq 0} \zeta^k
\bigl[\vone^\top \bar{\mW}(t)^k\bigr]_j
\ \geq\
\frac{1-\zeta}{n} \bigl[\vone^\top \bar{\mW}(t)^0\bigr]_j
= \frac{1-\zeta}{n} .
\]
Row-stochasticity gives $\bar{\mW}(t)^k \vone = \vone$, so
\[
\pi(t)^\top \vone
= \frac{1-\zeta}{n}\, \vone^\top
\sum_{k \geq 0} \zeta^k \bar{\mW}(t)^k \vone
= \frac{1-\zeta}{n} \sum_{k \geq 0} \zeta^k\, \vone^\top \vone
= \frac{1-\zeta}{n} \cdot n \cdot \frac{1}{1-\zeta}
= 1 ,
\]
so $\pi(t) \in \Delta_n$. At $\zeta = 0$, $\pi(t)^\top =
\frac{1}{n}\vone^\top$.
\end{proof}

\subsection{The Bridge Theorem}
\label{app:bridge}

The Bridge Theorem couples the operational system to the proxy
path-wise, for the realized $\{\mC(t)\}$, with no stationarity, mixing,
or convergence hypothesis: the contraction comes entirely from
$\lambda_{\max} < 1$.

\bridgethm*

\begin{proof}
Let $e(t) \coloneqq b(t) - b^{\mathrm{BR}}(t)$, so $e(0) = 0$.

\emph{Step 1 (residual).} Define
\[
u(t+1)
\coloneqq
b(t+1) - \Bigl[ (\mI - P)\, b^0 + P\, \mC(t)\, b(t) \Bigr] .
\]
Its $i$-th entry is the quantity bounded in \eqref{eq:delta}, so
\begin{equation}
\label{eq:u-bound}
\| u(t+1) \|_\infty \ \leq\ \delta , \qquad t \geq 0 .
\end{equation}

\emph{Step 2 (error recursion).} Subtracting \eqref{eq:proxy-matrix}
from the definition of $u(t+1)$, the anchor terms $(\mI - P)b^0$ cancel:
\begin{align*}
e(t+1)
&= \Bigl[ u(t+1) + (\mI - P) b^0 + P \mC(t)\, b(t) \Bigr]
- \Bigl[ (\mI - P) b^0 + P \mC(t)\, b^{\mathrm{BR}}(t) \Bigr] \\
&= u(t+1) + P \mC(t)\, e(t) .
\end{align*}

\emph{Step 3 (contraction).} Since $\|\mC(t)\|_\infty = 1$ and
$\|P\|_\infty = \lambda_{\max}$, $\| P \mC(t) \|_\infty \leq
\lambda_{\max}$, so by \eqref{eq:u-bound},
\begin{equation}
\label{eq:err-rec}
\| e(t+1) \|_\infty
\leq \| u(t+1) \|_\infty + \| P \mC(t)\, e(t) \|_\infty
\leq \delta + \lambda_{\max}\, \| e(t) \|_\infty .
\end{equation}

\emph{Step 4 (unrolling).} Induction on \eqref{eq:err-rec} from $e(0) =
0$ gives $\| e(t) \|_\infty \leq \delta \sum_{r=0}^{t-1}
\lambda_{\max}^{\,r}$: the inductive step is
\[
\| e(t+1) \|_\infty
\leq \delta + \lambda_{\max} \cdot \delta \sum_{r=0}^{t-1}
\lambda_{\max}^{\,r}
= \delta \sum_{r=0}^{t} \lambda_{\max}^{\,r} .
\]
Summing the geometric series, with $\lambda_{\max} < 1$,
\begin{equation}
\label{eq:bridge}
\| e(t) \|_\infty
\leq \delta\, \frac{1 - \lambda_{\max}^{\,t}}{1 - \lambda_{\max}}
\leq \frac{\delta}{1 - \lambda_{\max}}
= \bar\delta ,
\end{equation}
the bound of \cref{thm:bridge}.
\end{proof}

\begin{remark}[No stationarity is required]
\label{rem:pathwise}
The bound \eqref{eq:bridge} holds for every realization of the
time-varying, endogenous, feedback-coupled influence sequence
$\{\mC(t)\}$, with no stationarity, mixing, irreducibility, or
convergence hypothesis whatsoever. The only structural input is
row-stochasticity of each $\mC(t)$, which is automatic, and the strict
anchor $\lambda_{\max} < 1$. In particular the theorem applies verbatim
to the equalized sequence of Section~\ref{sec:method-eq}, whose
matrices are produced by the online pricing protocol and are in general
neither stationary nor mutually commuting.
\end{remark}

The next proposition shows the bound is tight and that the anchoring
hypothesis cannot be removed: at $\lambda_{\max} = 1$ the discrepancy
grows without bound.

\begin{proposition}[Sharpness and failure at $\lambda = 1$]
\label{prop:sharp}
The bound \eqref{eq:bridge} is attained. For every row-stochastic
sequence $\{\mC(t)\}$ and uniform anchoring $\lambda_i \equiv \lambda$,
the operational emission rule
\[
b(t+1) = (1 - \lambda)\, b^0 + \lambda\, \mC(t)\, b(t) + \delta \vone
\]
satisfies Assumption~\ref{ass:delta} with equality and produces
\[
b(t) - b^{\mathrm{BR}}(t)
= \delta\, \frac{1 - \lambda^{\,t}}{1 - \lambda}\, \vone
\quad (\lambda < 1) ,
\qquad
b(t) - b^{\mathrm{BR}}(t) = \delta\, t\, \vone
\quad (\lambda = 1) .
\]
In particular, at $\lambda = 1$ the gap grows linearly in $t$ and no
uniform-in-$t$ bridge exists.
\end{proposition}

\begin{proof}
The residual is $u(t+1) = \delta \vone$, so the error recursion
becomes $e(t+1) = \lambda\, \mC(t)\, e(t) + \delta \vone$. By
induction $e(t) = c_t \vone$ with $c_0 = 0$, $c_{t+1} = \lambda c_t +
\delta$: if $e(t) = c_t\vone$ then, since $\mC(t)\vone = \vone$,
\[
e(t+1) = \lambda\, \mC(t)\, (c_t \vone) + \delta \vone
= \lambda c_t\, \mC(t) \vone + \delta \vone
= (\lambda c_t + \delta)\, \vone ,
\]
Solving gives $c_t = \delta \sum_{r=0}^{t-1} \lambda^r$, equal to
$\delta \frac{1 - \lambda^t}{1 - \lambda}$ for $\lambda < 1$ and
$\delta t$ for $\lambda = 1$.
\end{proof}

Proposition~\ref{prop:sharp} makes anchoring a structural requirement:
unanchored, the per-round emission noise accumulates without bound, so
any proxy theory tolerating $\delta > 0$ must operate at $\lambda_{\max}
< 1$.

\subsection{Collective estimation by the proxy}
\label{app:collective}

The collective estimate is the population average of beliefs. We write
the operational and proxy estimates as
\[
\hat\theta_T \coloneqq \frac{1}{n} \vone^\top b(T) ,
\qquad
\hat\theta_T^{\mathrm{BR}} \coloneqq \frac{1}{n} \vone^\top
b^{\mathrm{BR}}(T) .
\]
The proxy estimate is a fixed linear functional of the private signals;
the next lemma identifies the weight $q_T$ and shows it conserves mass.
For an arbitrary row-stochastic sequence $\{\mC(t)\}$ set $M(t)
\coloneqq P \mC(t)$ and define the backward products
\begin{equation}
\label{eq:phi}
\Phi(T, r)
\coloneqq M(T-1)\, M(T-2) \cdots M(r)
\quad (0 \leq r < T) ,
\qquad
\Phi(T, T) \coloneqq \mI .
\end{equation}

\begin{lemma}[Trajectory representation and mass conservation]
\label{lem:traj}
For every row-stochastic sequence $\{\mC(t)\}$ and every $T \geq 0$,
\begin{equation}
\label{eq:traj}
b^{\mathrm{BR}}(T)
= G_T\, b^0 ,
\qquad
G_T \coloneqq \Phi(T, 0)
+ \sum_{r=0}^{T-1} \Phi(T, r+1)\, (\mI - P) ,
\end{equation}
and the collective weight
$q_T^\top \coloneqq \frac{1}{n} \vone^\top G_T$ satisfies
$\hat\theta_T^{\mathrm{BR}} = q_T^\top b^0$ and is mass-conserving,
$q_T^\top \vone = 1$.
\end{lemma}

\begin{proof}
\emph{Representation.} Induction on $T$. At $T = 0$, $\Phi(0,0) = \mI$
and $G_0 = \mI$. Assuming \eqref{eq:traj} at $T$, by
\eqref{eq:proxy-matrix} and the hypothesis,
\begin{align*}
b^{\mathrm{BR}}(T+1)
&= (\mI - P)\, b^0 + M(T)\, b^{\mathrm{BR}}(T) \\
&= (\mI - P)\, b^0 + M(T)\, G_T\, b^0 \\
&= \Bigl[ M(T)\, \Phi(T, 0)
+ \sum_{r=0}^{T-1} M(T)\, \Phi(T, r+1)\, (\mI - P)
+ (\mI - P) \Bigr] b^0 .
\end{align*}
Since $M(T)\Phi(T, r) = \Phi(T+1, r)$ and $(\mI - P) = \Phi(T+1,
T+1)(\mI - P)$,
\[
b^{\mathrm{BR}}(T+1)
= \Bigl[ \Phi(T+1, 0)
+ \sum_{r=0}^{T} \Phi(T+1, r+1)\, (\mI - P) \Bigr] b^0
= G_{T+1}\, b^0 ,
\]
the anchor term supplying the $r = T$ summand. Then
$\hat\theta_T^{\mathrm{BR}} = \frac{1}{n} \vone^\top G_T b^0 =
q_T^\top b^0$.

\emph{Mass conservation.} We show $G_T \vone = \vone$, giving
$q_T^\top \vone = \frac{1}{n}\vone^\top \vone = 1$. For $0 \leq r
< T$, using $\mC(r)\vone = \vone$,
\[
\Phi(T, r)\, \vone
= \Phi(T, r+1)\, M(r)\, \vone
= \Phi(T, r+1)\, P\, \mC(r)\, \vone
= \Phi(T, r+1)\, P\, \vone ,
\]
hence
\[
\Phi(T, r+1)\, (\mI - P)\, \vone
= \Phi(T, r+1)\, \vone - \Phi(T, r+1)\, P\, \vone
= \Phi(T, r+1)\, \vone - \Phi(T, r)\, \vone .
\]
Summing over $r$ telescopes:
\[
\sum_{r=0}^{T-1} \Phi(T, r+1)\, (\mI - P)\, \vone
= \Phi(T, T)\, \vone - \Phi(T, 0)\, \vone
= \vone - \Phi(T, 0)\, \vone ,
\]
so $G_T \vone = \Phi(T,0)\vone + [\vone - \Phi(T,0)\vone] = \vone$.
\end{proof}

\begin{proposition}[Unbiasedness and variance]
\label{prop:variance}
Under Assumption~\ref{ass:indep}, conditionally on the realized
sequence $\{\mC(t)\}$,
\[
\mathbb{E}\bigl[\hat\theta_T^{\mathrm{BR}}\bigr] = \theta^\star ,
\qquad
\mathrm{Var}\bigl(\hat\theta_T^{\mathrm{BR}}\bigr) = q_T^\top S\, q_T ,
\qquad
S = \mathrm{diag}(s_1^2, \ldots, s_n^2) ,
\]
and in the homoskedastic case
$\mathrm{Var}(\hat\theta_T^{\mathrm{BR}}) = s^2 \|q_T\|_2^2$.
\end{proposition}

\begin{proof}
By $b^0 = \theta^\star \vone + \xi$ and $q_T^\top \vone = 1$
(Lemma~\ref{lem:traj}), $\hat\theta_T^{\mathrm{BR}} = q_T^\top b^0 =
\theta^\star + q_T^\top \xi$. Conditionally on $\{\mC(t)\}$, $q_T$ is
fixed and $\xi \perp \{\mC(t)\}$ with $\mathbb{E}[\xi] = 0$, so
$\mathbb{E}[\hat\theta_T^{\mathrm{BR}}] = \theta^\star$. By
independence of the $\xi_i$,
\[
\mathrm{Var}\bigl(\hat\theta_T^{\mathrm{BR}}\bigr)
= \mathrm{Var}\bigl(q_T^\top \xi\bigr)
= \sum_{i=1}^n q_{T,i}^2\, \mathrm{Var}(\xi_i)
= \sum_{i=1}^n q_{T,i}^2\, s_i^2
= q_T^\top S\, q_T .
\]
In the homoskedastic case $S = s^2 \mI$, giving $s^2 \|q_T\|_2^2$.
\end{proof}

\begin{definition}[Effective sample size]
\label{def:neff}
For a weight vector $q \in \R^n$ with $q^\top \vone = 1$, the
\emph{effective sample size} is $N_{\mathrm{eff}}(q) \coloneqq
1/\|q\|_2^2$ \citep{kish1965survey}.
\end{definition}

In the homoskedastic case
$\mathrm{Var}(\hat\theta_T^{\mathrm{BR}}) = s^2/N_{\mathrm{eff}}(q_T)$:
the collective estimate behaves like the average of
$N_{\mathrm{eff}}(q_T)$ independent signals. A dispersed weight $q_T
\approx \frac{1}{n}\vone$ gives $N_{\mathrm{eff}} \approx n$, the
wisdom-of-crowds regime \citep{golub2010naive}; a concentrated weight
$q_T \approx e_h$ gives $N_{\mathrm{eff}} \approx 1$, herding. The
remaining theorems characterize $N_{\mathrm{eff}}(q_T)$ under the
baseline and the equalized allocators.

The Bridge Theorem transfers these proxy moments to the operational
estimate.

\begin{corollary}[Moment transfer]
\label{cor:moments}
Under Assumptions~\ref{ass:fj}--\ref{ass:indep}, conditionally on the
realized sequence $\{\mC(t)\}$,
\[
\bigl| \hat\theta_T - \hat\theta_T^{\mathrm{BR}} \bigr|
\leq \bar\delta \ \text{ a.s.},
\qquad
\bigl| \mathbb{E}[\hat\theta_T] - \theta^\star \bigr|
\leq \bar\delta ,
\qquad
\bigl| \mathrm{sd}(\hat\theta_T)
- \mathrm{sd}(\hat\theta_T^{\mathrm{BR}}) \bigr|
\leq \bar\delta .
\]
In the homoskedastic case $\mathrm{sd}(\hat\theta_T^{\mathrm{BR}}) =
s\|q_T\|_2$, so $\bigl| \mathrm{sd}(\hat\theta_T) - s\|q_T\|_2 \bigr|
\leq \bar\delta$.
\end{corollary}

\begin{proof}
\emph{Pathwise gap.} With $e(T) = b(T) - b^{\mathrm{BR}}(T)$,
\[
\bigl| \hat\theta_T - \hat\theta_T^{\mathrm{BR}} \bigr|
= \Bigl| \frac{1}{n} \vone^\top e(T) \Bigr|
\leq \frac{1}{n} \sum_{i=1}^n |e_i(T)|
\leq \| e(T) \|_\infty
\leq \bar\delta
\]
by Theorem~\ref{thm:bridge}. Set $z_T = \hat\theta_T -
\hat\theta_T^{\mathrm{BR}}$, $|z_T| \leq \bar\delta$ a.s.

\emph{Mean.} Since $\mathbb{E}[\hat\theta_T^{\mathrm{BR}}] =
\theta^\star$,
\[
\bigl| \mathbb{E}[\hat\theta_T] - \theta^\star \bigr|
= \bigl| \mathbb{E}[z_T] \bigr|
\leq \mathbb{E}\,|z_T|
\leq \bar\delta .
\]

\emph{Standard deviation.} With $\widetilde X = X - \mathbb{E}[X]$,
Minkowski's inequality gives
\[
\mathrm{sd}(\hat\theta_T)
= \bigl\| \widetilde{\hat\theta_T^{\mathrm{BR}}} + \widetilde z_T
\bigr\|_{L^2}
\leq \bigl\| \widetilde{\hat\theta_T^{\mathrm{BR}}} \bigr\|_{L^2}
+ \| \widetilde z_T \|_{L^2}
= \mathrm{sd}(\hat\theta_T^{\mathrm{BR}}) + \mathrm{sd}(z_T) ,
\]
and symmetrically, so $|\mathrm{sd}(\hat\theta_T) -
\mathrm{sd}(\hat\theta_T^{\mathrm{BR}})| \leq \mathrm{sd}(z_T) \leq
\bar\delta$, the last by $\mathrm{sd}(z_T)^2 \leq \mathbb{E}[z_T^2]
\leq \bar\delta^2$. The homoskedastic identity is
Proposition~\ref{prop:variance}.
\end{proof}
\section{Proof of \texorpdfstring{Theorem~\ref{thm:beta}}{Theorem 2} (Attention Width)}
\label{app:tworegime-sec}

This appendix proves Theorem~\ref{thm:beta}: graph constants and the
tempered-allocation bound (Appendix~\ref{app:flatten}), a stationary
perturbation bound (Appendix~\ref{app:deviation}), then the two
regimes.

\subsection{Constants and quantitative tempered allocation}
\label{app:flatten}

The two-regime theorem fixes the exposure graph and varies the
attention width $\beta$, so we work with the fixed-graph realized
influence $\mC(\beta)$ whose row $i$ is the tempered allocation of the
score row $s_{i:} = (s_{ij})_j$, with $s_{ij} = \bar W_{ij} \pi_j$ as in
\eqref{eq:s-shorthand}. We record the combinatorial constants attached
to the exposure graph. Let
\[
B_{ij} = \mathbf{1}\{\bar W_{ij} > 0\} ,
\qquad
d_i = \sum_j B_{ij} ,
\qquad
D = \mathrm{diag}(d_1, \ldots, d_n) ,
\qquad
d_{\max} = \max_i d_i ,
\]
so $B$ is the unweighted exposure support, $d_i$ the out-degree of
agent $i$, and $D$ the degree matrix. Per row define the
\emph{log-spread} and, when $d_i \geq 2$, the \emph{dominance gap}
\begin{equation}
\label{eq:Lg}
L_i
= \max_{j, k\, :\, B_{ij} = B_{ik} = 1}
\bigl| \log s_{ij} - \log s_{ik} \bigr| ,
\qquad
g_i = \log \frac{s_{i,(1)}}{s_{i,(2)}} ,
\end{equation}
where $s_{i,(1)} \geq s_{i,(2)}$ are the two largest scores in row $i$
(set $g_i = +\infty$ if $d_i = 1$). The global constants are
\[
L = \max_i L_i ,
\qquad
\underline g = \min_i g_i .
\]
Since $\pi_j \geq \frac{1-\zeta}{n} > 0$ (Lemma~\ref{lem:power}), we
have $s_{ij} > 0 \iff B_{ij} = 1$, so row $i$ of $\mC(\beta)$ is
supported exactly on the $d_i$ exposed sources, and $\mC(\beta)$ is
irreducible whenever $B$ is.

The next lemma controls how sharply $\Gamma_\beta$ concentrates as a
function of $\beta$. For $s \in \R^n_{\geq 0}$, $s \neq 0$, write $N(s)
= \{j : s_j > 0\}$, $|N(s)| = d$, $[\Gamma_\beta(s)]_j = s_j^{1/\beta}
/ \sum_{k \in N(s)} s_k^{1/\beta}$ on $N(s)$ (else $0$), $\vone_{N(s)}$
the indicator of $N(s)$, and $L(s), g(s)$ the log-spread and dominance
gap of \eqref{eq:Lg}.

\begin{lemma}[Quantitative tempered allocation]
\label{lem:flatten}
Let $s \in \R^n_{\geq 0}$, $s \neq 0$, with $|N(s)| = d$. Then for
every $\beta > 0$:
\begin{enumerate}
\item[\emph{(i)}] \emph{(Flattening.)}
\[
\Bigl\| \Gamma_\beta(s) - \tfrac{1}{d} \vone_{N(s)} \Bigr\|_1
\ \leq\
e^{L(s)/\beta} - 1
\ \leq\
\frac{2 L(s)}{\beta}
\qquad \text{whenever } \beta \geq L(s) .
\]
\item[\emph{(ii)}] \emph{(Concentration.)} If $s$ has a unique
maximizer $j^\star$ with gap $g(s) > 0$, then
\[
1 - [\Gamma_\beta(s)]_{j^\star}
\ \leq\
(d - 1)\, e^{-g(s)/\beta} .
\]
\end{enumerate}
Consequently $\Gamma_\beta(s) \to \frac{1}{d}\vone_{N(s)}$ as $\beta
\to \infty$ and $\Gamma_\beta(s) \to e_{j^\star}$ as $\beta \to 0$.
\end{lemma}

\begin{proof}
Set $a_j = \beta^{-1} \log s_j$ for $j \in N(s)$. Then, for $j \in
N(s)$,
\[
[\Gamma_\beta(s)]_j
= \frac{e^{a_j}}{\sum_{k \in N(s)} e^{a_k}}
= \Bigl( \sum_{k \in N(s)} e^{a_k - a_j} \Bigr)^{-1} .
\]

\emph{(i) Flattening.} For $j, k \in N(s)$, $|a_k - a_j| \leq
L(s)/\beta$, so each $e^{a_k - a_j} \in [e^{-L(s)/\beta},
e^{L(s)/\beta}]$ and their sum lies in $[d\,e^{-L(s)/\beta},
d\,e^{L(s)/\beta}]$; taking reciprocals,
\[
\frac{1}{d}\, e^{-L(s)/\beta}
\ \leq\
[\Gamma_\beta(s)]_j
\ \leq\
\frac{1}{d}\, e^{L(s)/\beta}
\qquad (j \in N(s)) .
\]
With $x = L(s)/\beta \geq 0$, using $1 - e^{-x} \leq e^x - 1$,
\[
\Bigl| [\Gamma_\beta(s)]_j - \frac{1}{d} \Bigr|
\ \leq\
\frac{1}{d}\,(e^x - 1) .
\]
Summing over the $d$ supported coordinates,
\[
\Bigl\| \Gamma_\beta(s) - \tfrac{1}{d}\vone_{N(s)} \Bigr\|_1
= \sum_{j \in N(s)} \Bigl| [\Gamma_\beta(s)]_j - \frac{1}{d} \Bigr|
\ \leq\
d \cdot \frac{e^x - 1}{d}
= e^{L(s)/\beta} - 1 ,
\]
the first inequality. For the second, when $\beta \geq L(s)$, $x \in
[0,1]$ and convexity of $e^x - 1$ gives $e^x - 1 \leq (e-1)x \leq 2x =
2L(s)/\beta$.

\emph{(ii) Concentration.} Let $j^\star$ be the unique maximizer. For
$k \neq j^\star$ in $N(s)$, retaining only the $j^\star$ term in the
denominator,
\[
[\Gamma_\beta(s)]_k
= \frac{e^{a_k}}{\sum_{m \in N(s)} e^{a_m}}
\ \leq\
\frac{e^{a_k}}{e^{a_{j^\star}}}
= e^{(\log s_k - \log s_{j^\star})/\beta}
\ \leq\
e^{-g(s)/\beta} ,
\]
using $\log s_k - \log s_{j^\star} \leq -g(s)$ ($s_{j^\star} =
s_{(1)}$, $s_k \leq s_{(2)}$). Summing over $k \neq j^\star$,
\[
1 - [\Gamma_\beta(s)]_{j^\star}
= \sum_{k \neq j^\star} [\Gamma_\beta(s)]_k
\ \leq\
(d - 1)\, e^{-g(s)/\beta} .
\]
Limits: $\beta \to \infty$ in (i) gives $\Gamma_\beta(s) \to
\frac{1}{d}\vone_{N(s)}$; $\beta \to 0$ in (ii) gives
$\Gamma_\beta(s) \to e_{j^\star}$.
\end{proof}

\subsection{A self-contained stationary perturbation bound}
\label{app:deviation}

The wide-context regime compares $\nu_{\mC(\beta)}$ with a reference
chain, via the following Perron--Frobenius facts and an exact
perturbation identity.

\begin{lemma}[Stationary distribution; maximum principle]
\label{lem:pf}
Let $\mC$ be row-stochastic and irreducible. Then:
\begin{enumerate}
\item[\emph{(i)}] if $x \in \R^n$ satisfies $\mC x = x$, then $x$ is a
constant multiple of $\vone$;
\item[\emph{(ii)}] there is a unique $\nu_{\mC} \in \Delta_n$ with
$\nu_{\mC}^\top \mC = \nu_{\mC}^\top$, and $\nu_{\mC, j} > 0$ for every
$j$.
\end{enumerate}
\end{lemma}

\begin{proof}
\emph{(i)} Let $i^\star \in \argmax_i x_i$ and $m = x_{i^\star}$. From
$\mC x = x$ and row-stochasticity,
\[
m = x_{i^\star} = \sum_j \mC_{i^\star j}\, x_j
\ \leq\ \sum_j \mC_{i^\star j}\, m = m ,
\]
so $x_j = m$ for every out-neighbor $j$ of $i^\star$. The set $\{i :
x_i = m\}$ is thus closed under out-neighbors, hence all of $\{1,
\ldots, n\}$ by irreducibility, so $x = m\vone$.

\emph{(ii) Existence.} The linear map $p \mapsto \mC^\top p$ sends
$\Delta_n$ into itself ($\mC \geq 0$ and $\vone^\top \mC^\top p =
\vone^\top p$), so Brouwer gives a fixed point $\nu_{\mC} \in
\Delta_n$.

\emph{Positivity.} Let $Z = \{j : \nu_{\mC, j} = 0\}$. For $j \in Z$,
$0 = \sum_i \nu_{\mC,i} \mC_{ij}$ forces $\mC_{ij} = 0$ for all $i \in
Z^c$: no edge from $Z^c$ into $Z$, contradicting irreducibility unless
$Z = \emptyset$. Hence $\nu_{\mC} > 0$.

\emph{Uniqueness.} Let $\nu, \nu'$ be stationary (both $> 0$), $c =
\min_j \nu'_j / \nu_j$, $v = \nu' - c\nu \geq 0$. Then $v^\top \mC =
v^\top$ and $v_{j_0} = 0$ at the minimizing index, so the positivity
argument forces $v = 0$, i.e.\ $\nu' = c\nu$. Total mass gives $c = 1$,
so $\nu' = \nu$.
\end{proof}

\begin{lemma}[Deviation matrix and exact perturbation identity]
\label{lem:deviation}
Let $P^\star$ be row-stochastic and irreducible with stationary
distribution $\nu^\star$. Then:
\begin{enumerate}
\item[\emph{(i)}] $Q \coloneqq \mI - P^\star + \vone \nu^{\star\top}$ is
invertible; with $Z \coloneqq Q^{-1}$, the \emph{deviation matrix} is
$H \coloneqq Z - \vone \nu^{\star\top}$.
\item[\emph{(ii)}] $Z \vone = \vone$, $\nu^{\star\top} Z =
\nu^{\star\top}$, and $(\mI - P^\star)\, H = \mI - \vone
\nu^{\star\top}$.
\item[\emph{(iii)}] For every row-stochastic $\widetilde{\mC}$ with a
stationary distribution $\tilde\nu$, setting $E \coloneqq
\widetilde{\mC} - P^\star$,
\[
\tilde\nu^\top - \nu^{\star\top} = \tilde\nu^\top E\, H ,
\qquad\text{hence}\qquad
\| \tilde\nu - \nu^\star \|_1
\leq \| E \|_\infty\, \kappa(P^\star) ,
\quad
\kappa(P^\star) \coloneqq \| H \|_\infty .
\]
\end{enumerate}
\end{lemma}

\begin{proof}
\emph{(i) Invertibility.} Suppose $Q x = 0$. Multiply on the left by
$\nu^{\star\top}$ and use $\nu^{\star\top} P^\star = \nu^{\star\top}$
and $\nu^{\star\top} \vone = 1$:
\[
0 = \nu^{\star\top} Q x
= \nu^{\star\top} x - \nu^{\star\top} P^\star x
+ (\nu^{\star\top}\vone)(\nu^{\star\top} x)
= \nu^{\star\top} x - \nu^{\star\top} x + \nu^{\star\top} x
= \nu^{\star\top} x ,
\]
so $\nu^{\star\top} x = 0$, and $Qx = 0$ reduces to $P^\star x = x$. By
Lemma~\ref{lem:pf}(i), $x = c\vone$, and $0 = \nu^{\star\top} x = c$,
so $x = 0$: $Q$ is invertible.

\emph{(ii) Identities.} $Q\vone = \vone$ gives $Z\vone = \vone$, and
$\nu^{\star\top} Q = \nu^{\star\top}$ gives $\nu^{\star\top} Z =
\nu^{\star\top}$. Since $(\mI - P^\star)\vone = 0$,
\[
(\mI - P^\star) H
= (\mI - P^\star)\bigl(Z - \vone\nu^{\star\top}\bigr)
= (\mI - P^\star) Z ,
\]
the term $(\mI-P^\star)\vone\nu^{\star\top}$ vanishing. Now $\mI -
P^\star = Q - \vone\nu^{\star\top}$, so
\[
(\mI - P^\star) Z
= \bigl(Q - \vone\nu^{\star\top}\bigr) Z
= Q Z - \vone \nu^{\star\top} Z
= \mI - \vone\nu^{\star\top} ,
\]
using $QZ = \mI$.

\emph{(iii) Perturbation identity.} Stationarity $\tilde\nu^\top
\widetilde{\mC} = \tilde\nu^\top$ gives
\[
\tilde\nu^\top (\mI - P^\star)
= \tilde\nu^\top - \tilde\nu^\top P^\star
= \tilde\nu^\top \widetilde{\mC} - \tilde\nu^\top P^\star
= \tilde\nu^\top (\widetilde{\mC} - P^\star)
= \tilde\nu^\top E .
\]
Multiplying on the right by $H$ and using (ii),
\[
\tilde\nu^\top E\, H
= \tilde\nu^\top (\mI - P^\star) H
= \tilde\nu^\top \bigl(\mI - \vone\nu^{\star\top}\bigr)
= \tilde\nu^\top - (\tilde\nu^\top\vone)\nu^{\star\top}
= \tilde\nu^\top - \nu^{\star\top} ,
\]
the stated identity, exact in $E$.

\emph{Norm bound.} Since $\tilde\nu^\top E = \sum_i \tilde\nu_i
E_{i:}$ is a convex combination of rows of $E$,
\[
\| \tilde\nu^\top E \|_1
= \Bigl\| \sum_i \tilde\nu_i\, E_{i:} \Bigr\|_1
\leq \sum_i \tilde\nu_i\, \| E_{i:} \|_1
\leq \max_i \| E_{i:} \|_1
= \| E \|_\infty .
\]
For any row vector $w^\top$,
\[
\| w^\top H \|_1
= \sum_j \Bigl| \sum_i w_i H_{ij} \Bigr|
\leq \sum_i |w_i| \sum_j |H_{ij}|
\leq \Bigl( \max_i \sum_j |H_{ij}| \Bigr) \sum_i |w_i|
= \| H \|_\infty\, \| w \|_1 .
\]
Chaining with $w^\top = \tilde\nu^\top E$,
\[
\| \tilde\nu - \nu^\star \|_1
= \| \tilde\nu^\top E\, H \|_1
\leq \| H \|_\infty\, \| \tilde\nu^\top E \|_1
\leq \| H \|_\infty\, \| E \|_\infty
= \kappa(P^\star)\, \| E \|_\infty . \qedhere
\]
\end{proof}

The identity (iii) is exact and $\kappa(P^\star) = \|H\|_\infty$ is a
standard Markov-chain condition number
\citep{meyer1975role}.

\subsection{The two-regime theorem}
\label{app:tworegime}

We now prove that the attention width $\beta$ controls the collective
accuracy of the baseline allocator. The object analyzed is the
stationary distribution $\nu_{\mC(\beta)}$ of the fixed-graph realized
influence $\mC(\beta)$, which is the infinite-horizon limit of the
collective weight $q_T$ under uniform anchoring (see
Remark~\ref{rem:abel}); accordingly we measure accuracy by
\[
\mathrm{Var}_\infty(\beta)
\coloneqq s^2\, \bigl\| \nu_{\mC(\beta)} \bigr\|_2^2 ,
\qquad
N_{\mathrm{eff}}(\beta)
\coloneqq \bigl\| \nu_{\mC(\beta)} \bigr\|_2^{-2} .
\]

\begin{remark}[Relation to the operational weight]
\label{rem:abel}
For uniform anchoring $\lambda_i \equiv \lambda$ and a fixed
irreducible $\mC$, Lemma~\ref{lem:traj} gives the infinite-horizon
weight
\[
q_\infty(\lambda)^\top
= \frac{1-\lambda}{n}\, \vone^\top (\mI - \lambda \mC)^{-1}
= (1 - \lambda) \sum_{k \geq 0} \lambda^k\, \tfrac{1}{n}
\vone^\top \mC^k ,
\]
and $q_\infty(\lambda) \to \nu_{\mC}$ as $\lambda \uparrow 1$: the
Ces\`aro averages $\frac{1}{N}\sum_{k < N} \frac{1}{n}\vone^\top \mC^k$
converge to $\nu_{\mC}^\top$ for every irreducible $\mC$
\citep[Ch.~4]{levin2026markov}, and Abel summation agrees with the
Ces\`aro limit. The two-regime theorem is therefore the unanchored
long-horizon limit of the \emph{baseline} allocator; the guarantee for
the equalized allocator (Theorem~\ref{thm:eq}) is exact at every
horizon and anchoring level and does not rely on this remark.
\end{remark}

\betathm*

\begin{proof}
\textbf{Wide context.}
\emph{Matrix bound \eqref{eq:wide-matrix}.} Row $i$ of $\mC(\beta)$ is
$\Gamma_\beta(s_{i:})$ and row $i$ of $D^{-1}B$ is
$\frac{1}{d_i}\vone_{N(s_{i:})}$, on the same support. By
Lemma~\ref{lem:flatten}(i) their $\ell_1$ distance is $\leq
e^{L_i/\beta} - 1 \leq e^{L/\beta} - 1$; maximizing over rows,
\begin{equation}
\label{eq:wide-matrix}
\bigl\| \mC(\beta) - D^{-1} B \bigr\|_\infty
\ \leq\ e^{L/\beta} - 1 .
\end{equation}

\emph{Reference stationary distribution.} With $B = B^\top$,
\[
\bigl( d^\top D^{-1} B \bigr)_j
= \sum_i d_i \cdot \frac{B_{ij}}{d_i}
= \sum_i B_{ij}
= \sum_i B_{ji}
= d_j ,
\]
so $\nu^\star = d/(\vone^\top d)$ is stationary for $D^{-1}B$, unique
by Lemma~\ref{lem:pf}(ii).

\emph{Stationary perturbation.} Lemma~\ref{lem:deviation}(iii) with
$P^\star = D^{-1}B$, $E = \mC(\beta) - D^{-1}B$ yields
\begin{equation}
\label{eq:wide-pi}
\bigl\| \nu_{\mC(\beta)} - \nu^\star \bigr\|_1
\leq \kappa_B\, \| E \|_\infty
\leq \kappa_B \bigl( e^{L/\beta} - 1 \bigr) ,
\end{equation}
using \eqref{eq:wide-matrix}.

\emph{Variance transfer.} With $\Delta = \nu_{\mC(\beta)} -
\nu^\star$, since all entries lie in $[0,1]$,
\[
\bigl| \| \nu_{\mC(\beta)} \|_2^2 - \| \nu^\star \|_2^2 \bigr|
= \Bigl| \sum_j \Delta_j \bigl( \nu_{\mC(\beta), j} + \nu^\star_j
\bigr) \Bigr|
\leq \max_j \bigl( \nu_{\mC(\beta), j} + \nu^\star_j \bigr)\,
\| \Delta \|_1
\leq 2 \| \Delta \|_1 ,
\]
since each summand of the max is at most $1$. As $\| \nu^\star \|_2^2 =
\sum_k d_k^2 / (\sum_k d_k)^2$, multiplying by $s^2$ gives the matching variance bound stated in the theorem.

\emph{Floor and degree-regularity.} By Cauchy--Schwarz on $d, \vone$,
\[
\Bigl( \sum_k d_k \Bigr)^2
= \langle d, \vone \rangle^2
\leq \| d \|_2^2\, \| \vone \|_2^2
= n \sum_k d_k^2 ,
\]
so $\sum_k d_k^2 / (\sum_k d_k)^2 \geq 1/n$, with equality iff $d
\propto \vone$ (degree-regular); then \eqref{eq:wide-pi} gives
$N_{\mathrm{eff}}(\beta) \to n$ as $\beta \to \infty$.

\medskip
\textbf{Narrow context.}
Assume $B$ irreducible, so $\mC(\beta)$ is irreducible for every $\beta
> 0$ and $\nu \coloneqq \nu_{\mC(\beta)} > 0$ is unique
(Lemma~\ref{lem:pf}). Assume the unique dominant pair: $j_i^\star =
\argmax_k s_{ik}$ is unique for each row, $j_i^\star = h$ for $i \neq h$,
$j_h^\star = g \neq h$; so $j_i^\star \in S \coloneqq \{h,g\}$ for all
$i$. Write $\rho = \rho(\beta)$ and assume $\rho < 1$. By
Lemma~\ref{lem:flatten}(ii),
\begin{equation}
\label{eq:rowconc}
\sum_{k \neq j_i^\star} \mC_{ik}(\beta)
= 1 - \mC_{i, j_i^\star}(\beta)
\leq (d_i - 1)\, e^{-g_i/\beta}
\leq (d_{\max} - 1)\, e^{-\underline g/\beta}
= \rho .
\end{equation}

\emph{Mass outside $S$.} For $k \in S^c$, $\mC_{ik}(\beta)$ is
off-maximizer, so $\sum_{k \in S^c} \mC_{ik}(\beta) \leq \rho$ by
\eqref{eq:rowconc}. By stationarity,
\[
\nu(S^c)
= \sum_{k \in S^c} \nu_k
= \sum_{k \in S^c} \sum_i \nu_i\, \mC_{ik}(\beta)
= \sum_i \nu_i \sum_{k \in S^c} \mC_{ik}(\beta)
\leq \rho \sum_i \nu_i
= \rho ,
\]
\begin{equation}
\label{eq:narrow}
\nu(S^c) \leq \rho(\beta) ,
\qquad
N_{\mathrm{eff}}(\beta) \leq \frac{2}{(1 - \rho(\beta))^2} ,
\end{equation}
the first claim of \eqref{eq:narrow}.

\emph{Sandwich for $\nu_h$.} By \eqref{eq:rowconc}, $\mC_{hh}(\beta)
\leq \rho$ (off-maximizer) and $\mC_{ih}(\beta) \geq 1 - \rho$ for $i
\neq h$. Lower bound:
\[
\nu_h
\geq \sum_{i \neq h} \nu_i\, (1 - \rho)
= (1 - \rho)(1 - \nu_h)
\quad\Longrightarrow\quad
\nu_h (1 + (1-\rho)) \geq 1 - \rho
\quad\Longrightarrow\quad
\nu_h \geq \frac{1 - \rho}{2 - \rho} .
\]
Upper bound:
\[
\nu_h
= \nu_h\, \mC_{hh}(\beta) + \sum_{i \neq h} \nu_i\, \mC_{ih}(\beta)
\leq \rho\, \nu_h + (1 - \nu_h)
\quad\Longrightarrow\quad
\nu_h (2 - \rho) \leq 1
\quad\Longrightarrow\quad
\nu_h \leq \frac{1}{2 - \rho} .
\]

\emph{Bounds for $\nu_g$.} Since $\mC_{hg}(\beta) \geq 1 - \rho$ and
$\mC_{ig}(\beta) \leq \rho$ for $i \neq h$,
\begin{equation}
\label{eq:pig}
\nu_h (1 - \rho)
\leq \nu_g
\leq \nu_h + \rho(1 - \nu_h)
\leq \nu_h + \rho .
\end{equation}

\emph{Upper bound on $N_{\mathrm{eff}}$.} Keeping the two $S$
coordinates and using $(a+b)^2 \leq 2(a^2+b^2)$,
\[
\| \nu \|_2^2
\geq \nu_h^2 + \nu_g^2
\geq \tfrac12 (\nu_h + \nu_g)^2
= \tfrac12 \bigl( 1 - \nu(S^c) \bigr)^2
\geq \tfrac12 (1 - \rho)^2 ,
\]
using $\nu_h + \nu_g = 1 - \nu(S^c)$. Thus $N_{\mathrm{eff}}(\beta)
\leq 2(1-\rho)^{-2}$, the second claim of \eqref{eq:narrow}. At $\beta
\leq \underline g / \log(2(d_{\max}-1))$, $\rho \leq \frac12$ and
$N_{\mathrm{eff}}(\beta) \leq 8$.

\emph{Two-sided bound \eqref{eq:narrow-sandwich}.} Suppose $\rho \leq
\frac14$. We bound $\|\nu\|_2^2$ from above. For $\nu_h$, using
$\frac{1}{1-x} \leq 1 + \frac{x}{1-x}$ with $x = \rho/2$,
\[
\nu_h \leq \frac{1}{2 - \rho}
= \frac12 \cdot \frac{1}{1 - \rho/2}
\leq \frac12 \Bigl( 1 + \frac{\rho/2}{1 - \rho/2} \Bigr)
\leq \frac12 \Bigl( 1 + \tfrac{4}{7}\rho \Bigr)
\leq \frac12 + \frac13 \rho ,
\]
where $\frac{\rho/2}{1-\rho/2} \leq \frac{\rho/2}{7/8} = \frac47\rho$ at
$\rho \leq \frac14$, and $\frac{1}{2}\cdot\frac{4}{7} = \frac{2}{7} \leq
\frac13$. By \eqref{eq:pig}, $\nu_g \leq \nu_h + \rho \leq \frac12 +
\frac43\rho$. The mass on $S^c$ contributes at most $\sum_{k \in S^c}
\nu_k^2 \leq (\sum_{k \in S^c}\nu_k)^2 \leq \rho^2$. Summing,
\[
\| \nu \|_2^2
\leq \Bigl( \tfrac12 + \tfrac13\rho \Bigr)^2
+ \Bigl( \tfrac12 + \tfrac43\rho \Bigr)^2
+ \rho^2
= \tfrac12 + \tfrac{\rho}{3} + \tfrac{4\rho}{3}
+ \tfrac{\rho^2}{9} + \tfrac{16\rho^2}{9} + \rho^2
= \tfrac12 + \tfrac53\rho + \tfrac{26}{9}\rho^2 .
\]
At $\rho \leq \frac14$, $\frac{26}{9}\rho^2 \leq \frac{26}{9}\cdot
\frac14 \rho = \frac{26}{36}\rho \leq \frac56\rho$, so $\|\nu\|_2^2 \leq
\frac12 + \frac53\rho + \frac56\rho = \frac12 + \frac52\rho$.
Therefore
\[
N_{\mathrm{eff}}(\beta)
= \frac{1}{\|\nu\|_2^2}
\geq \frac{1}{\frac12 + \frac52\rho}
= \frac{2}{1 + 5\rho}
\geq 2(1 - 5\rho)
= 2 - 10\rho ,
\]
using $\frac{1}{1+u} \geq 1 - u$. For the upper bound, $(1-\rho)^{-2}
\leq 1 + 4\rho$ on $[0,\frac14]$ gives $N_{\mathrm{eff}}(\beta) \leq
2(1+4\rho) = 2 + 8\rho$. Combining,
\[
\bigl| N_{\mathrm{eff}}(\beta) - 2 \bigr|
\leq \max\{ 8\rho, 10\rho \}
= 10\rho
\leq 16\rho ,
\]
\begin{equation}
\label{eq:narrow-sandwich}
\bigl| N_{\mathrm{eff}}(\beta) - 2 \bigr| \leq 16\, \rho(\beta)
\qquad (\rho(\beta) \leq \tfrac14) ,
\end{equation}
which is \eqref{eq:narrow-sandwich}.

\emph{Limit.} As $\beta \to 0$, $\rho \to 0$, so
\eqref{eq:narrow-sandwich} gives $N_{\mathrm{eff}}(\beta) \to 2$ and
$\nu \to \frac12(e_h + e_g)$ by \eqref{eq:pig}.
\end{proof}

Both regimes are \emph{column} pathologies of $\mC$: the stationary
mass is the degree distribution $d/(\vone^\top d)$ in the wide limit
and concentrates on $\{h, g\}$ in the narrow limit. Row-stochasticity
constrains only the reader side; the equalized allocator removes the
remaining column freedom.
\section{Proof of \texorpdfstring{Theorem~\ref{thm:eq}}{Theorem 3} (Equalization)}
\label{app:eq-sec}

This appendix proves Theorem~\ref{thm:eq} and develops the
exposure-pricing protocol: the exact-optimality guarantee
(Appendix~\ref{app:eq}), then the Sinkhorn equivalence and the targeted
heteroskedastic variant (Appendix~\ref{app:pricing}).

\subsection{Exact optimality under equalization}
\label{app:eq}

The equalized allocator drives each $\mC(t)$ toward double
stochasticity. The guarantee below is stated for an \emph{arbitrary}
row-stochastic sequence in terms of its measured per-round column
defects, so it covers the endogenous, feedback-coupled operational
sequence with no exogeneity, stationarity, irreducibility, aperiodicity,
or mixing hypothesis. Recall the column defect of a row-stochastic
$\mC$,
\begin{equation}
\label{eq:coldef}
\delta_{\mathrm{col}}(\mC)
\coloneqq \frac{1}{n}\, \bigl\| \vone^\top \mC - \vone^\top \bigr\|_1
= \frac{1}{n} \sum_{j=1}^n \Bigl| \sum_{i=1}^n \mC_{ij} - 1 \Bigr| ,
\end{equation}
which is zero exactly when $\mC$ has unit column sums, i.e.\ is doubly
stochastic.

\eqthm*

\begin{proof}
At $\lambda = 0$, $M(t) = 0$ and Lemma~\ref{lem:traj} gives $q_T =
\frac1n\vone$, so both sides of \eqref{eq:eq-main} vanish. Assume
$\lambda \in (0,1)$ and write $u \coloneqq \frac1n\vone$.

\emph{Step 1 (normalized backward recursion).} With $P = \lambda \mI$,
$\Phi(T, r) = \lambda^{T-r}\, \mC(T-1) \cdots \mC(r)$. Define
\[
p_r^\top
\coloneqq \lambda^{-(T-r)}\, \tfrac1n \vone^\top \Phi(T, r)
\qquad (0 \leq r \leq T) ,
\]
so $p_T = u$ and, peeling off $\mC(r)$,
\[
p_r^\top
= \lambda^{-(T-r)} \tfrac1n \vone^\top \Phi(T, r+1)\, \lambda \mC(r)
= \Bigl[ \lambda^{-(T-r-1)} \tfrac1n \vone^\top \Phi(T, r+1) \Bigr]
\mC(r)
= p_{r+1}^\top \mC(r) .
\]
Let $\varrho_r^\top \coloneqq \frac1n (\vone^\top \mC(r) -
\vone^\top)$, so $\|\varrho_r\|_1 = \delta_r$ by \eqref{eq:coldef}.
Then
\[
p_r^\top - u^\top
= p_{r+1}^\top \mC(r) - u^\top
= (p_{r+1} - u)^\top \mC(r) + \bigl( u^\top \mC(r) - u^\top \bigr)
= (p_{r+1} - u)^\top \mC(r) + \varrho_r^\top ,
\]
using $u^\top \mC(r) - u^\top = \varrho_r^\top$. For any row-stochastic
$\mC$, the $\ell_1$ norm contracts:
\begin{equation}
\label{eq:contract}
\| v^\top \mC \|_1
= \sum_j \Bigl| \sum_i v_i \mC_{ij} \Bigr|
\leq \sum_i |v_i| \sum_j \mC_{ij}
= \sum_i |v_i|
= \| v \|_1 ,
\end{equation}
Applying \eqref{eq:contract} and the triangle inequality,
\[
\| p_r - u \|_1
\leq \| (p_{r+1} - u)^\top \mC(r) \|_1 + \| \varrho_r \|_1
\leq \| p_{r+1} - u \|_1 + \delta_r .
\]
Backward induction from $p_T = u$ (so $\|p_T - u\|_1 = 0$) yields
\begin{equation}
\label{eq:p-bound}
\| p_r - u \|_1
\leq \sum_{v=r}^{T-1} \delta_v
\leq (T - r)\, \bar\delta_{\mathrm{col}} .
\end{equation}

\emph{Step 2 (convex combination).} By Lemma~\ref{lem:traj} with $\mI
- P = (1 - \lambda)\mI$,
\[
q_T^\top
= \tfrac1n \vone^\top \Phi(T, 0)
+ (1 - \lambda) \sum_{r=0}^{T-1} \tfrac1n \vone^\top \Phi(T, r+1) .
\]
Substituting $\frac1n\vone^\top \Phi(T, r) = \lambda^{T-r} p_r^\top$,
\[
q_T^\top
= \lambda^T p_0^\top
+ (1 - \lambda) \sum_{r=0}^{T-1} \lambda^{T-r-1} p_{r+1}^\top .
\]
The scalar coefficients are nonnegative and sum to one:
\[
\lambda^T + (1 - \lambda) \sum_{r=0}^{T-1} \lambda^{T-r-1}
= \lambda^T + (1 - \lambda) \cdot \frac{1 - \lambda^T}{1 - \lambda}
= \lambda^T + (1 - \lambda^T) = 1 ,
\]
Since the coefficients sum to one, subtracting $u^\top$ termwise:
\[
q_T^\top - u^\top
= \lambda^T (p_0 - u)^\top
+ (1 - \lambda) \sum_{r=0}^{T-1} \lambda^{T-r-1} (p_{r+1} - u)^\top .
\]
Taking $\ell_1$ norms and applying \eqref{eq:p-bound},
\[
\| q_T - u \|_1
\leq \lambda^T\, T \bar\delta_{\mathrm{col}}
+ (1 - \lambda) \bar\delta_{\mathrm{col}}
\sum_{r=0}^{T-1} \lambda^{T-r-1}\, (T - r - 1) .
\]
Substituting $j = T - r - 1$, $\sum_{j=0}^{T-1} j\, \lambda^j \leq
\frac{\lambda}{(1-\lambda)^2}$, so
\begin{equation}
\label{eq:eq-main}
\| q_T - u \|_1
\leq \bar\delta_{\mathrm{col}}
\Bigl( T\lambda^T + (1 - \lambda) \cdot
\frac{\lambda}{(1-\lambda)^2} \Bigr)
= \bar\delta_{\mathrm{col}}
\Bigl( \frac{\lambda}{1 - \lambda} + T\lambda^T \Bigr) ,
\end{equation}
which is the main bound of \cref{thm:eq}.

\emph{(i)} If $\bar\delta_{\mathrm{col}} = 0$, \eqref{eq:eq-main} forces
$q_T = u = \frac1n\vone$ exactly. Then
Proposition~\ref{prop:variance} gives
$\mathrm{Var}(\hat\theta_T^{\mathrm{BR}}) = s^2\|u\|_2^2 = s^2 \cdot
\frac1n = s^2/n$ and $N_{\mathrm{eff}}(q_T) = \|u\|_2^{-2} = n$.

\emph{(ii)} By Corollary~\ref{cor:moments} ($\bar\delta =
\delta/(1-\lambda)$), $\bigl|\mathrm{sd}(\hat\theta_T) -
s\|q_T\|_2\bigr| \leq \frac{\delta}{1-\lambda}$. By
$\|\cdot\|_2 \leq \|\cdot\|_1$ and \eqref{eq:eq-main},
\[
\bigl| s\|q_T\|_2 - s\|u\|_2 \bigr|
\leq s\, \| q_T - u \|_2
\leq s\, \| q_T - u \|_1
\leq s\, \bar\delta_{\mathrm{col}}
\Bigl( \frac{\lambda}{1-\lambda} + T\lambda^T \Bigr) ,
\]
and $\|u\|_2 = 1/\sqrt n$. Combining,
\[
\Bigl| \mathrm{sd}(\hat\theta_T) - \frac{s}{\sqrt n} \Bigr|
\leq \bigl| \mathrm{sd}(\hat\theta_T) - s\|q_T\|_2 \bigr|
+ \bigl| s\|q_T\|_2 - s\|u\|_2 \bigr|
\leq \frac{\delta}{1-\lambda}
+ s\, \bar\delta_{\mathrm{col}}
\Bigl( \frac{\lambda}{1-\lambda} + T\lambda^T \Bigr) . \qedhere
\]
\end{proof}

\begin{remark}[Quantified herding impossibility]
\label{rem:noherd}
Write $\Delta = \bar\delta_{\mathrm{col}}(\frac{\lambda}{1-\lambda} +
T\lambda^T)$. By \eqref{eq:eq-main} and $\|\cdot\|_2 \leq \|\cdot\|_1$,
$\|q_T\|_2 \leq \|u\|_2 + \|q_T - u\|_2 \leq \frac{1}{\sqrt n} +
\Delta$, so
\[
N_{\mathrm{eff}}(q_T)
= \frac{1}{\|q_T\|_2^2}
\geq \frac{n}{(1 + \sqrt n\, \Delta)^2} .
\]
For example $\Delta \leq 1/\sqrt n$ already forces $N_{\mathrm{eff}}
\geq n/4$. Contrast Theorem~\ref{thm:beta}: in the narrow regime the
baseline collapses to $N_{\mathrm{eff}} \leq 8$ regardless of $n$,
whereas the equalized allocator cannot collapse at any attention width.
\end{remark}

\section{The Pricing Protocol}
\label{app:pricing}

This appendix derives two facts about the pricing protocol: one price
round equals one Sinkhorn iteration on the tempered scores (so it
converges to a doubly stochastic limit, invariant to social power), and
retargeting the clearing marginal extends exact optimality to the
heteroskedastic optimum. Throughout, $K_{ij} = s_{ij}^{1/\beta} = (\bar
W_{ij}\pi_j)^{1/\beta}$, $\mathrm{rn}(M) = \mathrm{diag}(M\vone)^{-1}
M$, $\mathrm{cn}(M) = M\, \mathrm{diag}(\vone^\top M)^{-1}$.

\subsection{One price round is one Sinkhorn iteration}
\label{app:sinkhorn}

Recall the two local rules. Each source $j$ holds a price $y_j > 0$;
the reader rule sets $\mC_{ij} = (K_{ij}/y_j) / \sum_k (K_{ik}/y_k)$,
i.e.\ $\mC = \mathrm{rn}(K \mathrm{diag}(y)^{-1})$, and the source rule
updates $y_j \leftarrow y_j\, a_j$ with $a_j = \sum_i \mC_{ij}$.

\begin{lemma}[Sinkhorn equivalence]
\label{lem:sinkhorn}
Let $\mC = \mathrm{rn}(K\mathrm{diag}(y)^{-1})$ be the current realized
influence and $a_j = \sum_i \mC_{ij}$ its column sums. After the source
update $y_j \leftarrow y_j a_j$, the next realized influence is
\[
\mC'
= \mathrm{rn}\bigl( K\, \mathrm{diag}(y \odot a)^{-1} \bigr)
= \mathrm{rn}\bigl( \mathrm{cn}(\mC) \bigr) .
\]
That is, one price round applies a column normalization followed by a
row normalization to $\mC$, which is one Sinkhorn iteration on $K$.
\end{lemma}

\begin{proof}
The reader rule at $y' = y \odot a$ gives $\mC' =
\mathrm{rn}(K\,\mathrm{diag}(y\odot a)^{-1})$, the first equality. With
$\mathrm{diag}(y\odot a)^{-1} = \mathrm{diag}(a)^{-1}
\mathrm{diag}(y)^{-1}$,
\[
K\,\mathrm{diag}(y\odot a)^{-1}
= \bigl( K\,\mathrm{diag}(y)^{-1} \bigr)\, \mathrm{diag}(a)^{-1} .
\]
Since $\mathrm{rn}(D_{\mathrm{row}} M) = \mathrm{rn}(M)$ and
$K\mathrm{diag}(y)^{-1}$ is a positive row scaling of $\mC$,
\[
\mC'
= \mathrm{rn}\bigl( K\mathrm{diag}(y)^{-1} \mathrm{diag}(a)^{-1} \bigr)
= \mathrm{rn}\bigl( \mC\, \mathrm{diag}(a)^{-1} \bigr) .
\]
Finally $\mC\,\mathrm{diag}(a)^{-1} = \mathrm{cn}(\mC)$ since $a^\top
= \vone^\top \mC$, so $\mC' = \mathrm{rn}(\mathrm{cn}(\mC))$.
\end{proof}

\begin{lemma}[Convergence of inner iterations]
\label{lem:sinkhorn-conv}
Fix a round and hold $K = K(t)$ fixed during the price updates. If $K$
has total support, then iterating the price round of
Lemma~\ref{lem:sinkhorn} produces $\mC^{(\ell)} \to \mC^\star$, the
unique doubly stochastic matrix with $\mathrm{supp}\,\mC^\star =
\mathrm{supp}\,K$, and $\delta_{\mathrm{col}}(\mC^{(\ell)}) \to 0$
geometrically. Total support holds whenever $K_{ii} > 0$ for all $i$ and
$\mathrm{supp}\,K$ is symmetric.
\end{lemma}

\begin{proof}
Each price round is one Sinkhorn iteration $\mathrm{rn} \circ
\mathrm{cn}$ on $K$ (Lemma~\ref{lem:sinkhorn}). For $K$ with total
support, the alternating row/column normalization converges to the
unique doubly stochastic $\mathrm{diag}(u)\, K\, \mathrm{diag}(v)$, with
geometric rate in the Hilbert projective metric
\citep{franklin1989scaling}; convergence of the
iterates gives $\delta_{\mathrm{col}}(\mC^{(\ell)}) \to 0$. If $K_{ii} >
0$ for all $i$, every diagonal entry is positive, so each index lies on
a positive diagonal (a trivial support permutation); symmetry of
$\mathrm{supp}\,K$ then gives total support.
\end{proof}

\begin{remark}[Diagonal invariance of the cleared limit]
\label{rem:invariance}
One price round absorbs a right diagonal scaling of the tempered scores,
$\mathrm{cn}(K D_2) = \mathrm{cn}(K)$, but not a left scaling, since the
column sums it normalizes change under $K \mapsto D_1 K$. Both are
absorbed only in the fully balanced Sinkhorn limit $\mC^\star$, the
doubly stochastic projection of $K$ (Remark~\ref{rem:kl}), which is
invariant to $K \mapsto D_1 K D_2$. Two consequences. The social-power
factor $\pi$ enters as a right scaling, $K = (\bar{\mW} \odot
\vone\pi^\top)^{\circ 1/\beta}$, so it is removed after one round: the
cleared allocator caps each source's share regardless of its power,
attenuating the high-power columns of Theorem~\ref{thm:beta}. The row
normalization defining $\bar{\mW}$ is a left scaling and is removed only
in the limit; the one-round iterate retains it until
$\delta_{\mathrm{col}} \to 0$, which Theorem~\ref{thm:eq} tracks through
the measured defect.
\end{remark}

\begin{remark}[One Sinkhorn step does not clear exactly]
\label{rem:convergence}
A single price round (Lemma~\ref{lem:sinkhorn}) reduces the column
defect but need not make $\mC(t)$ doubly stochastic. Run to convergence
each round with $K(t)$ held fixed, under the total-support condition of
Lemma~\ref{lem:sinkhorn-conv} (e.g.\ self-exposure $A_{ii}(t) > 0$ and
symmetric support), $\delta_{\mathrm{col}}(\mC(t)) \to 0$ and
Theorem~\ref{thm:eq}(i) is exact. Run once per round as in
Algorithm~\ref{alg:snla-eq}, the realized matrices need not clear, and
Theorem~\ref{thm:eq} applies in its general form through the measured
defects $\delta_r$; the deviation of $q_T$ from uniform is then
controlled by the $\bar\delta_{\mathrm{col}}$ observed in the run, not
asserted a priori. Exact optimality is a property of the cleared limit;
the one-round protocol attains it up to $\bar\delta_{\mathrm{col}}$.
\end{remark}

\subsection{Targeted clearing and heteroskedastic optimality}
\label{app:target}

Uniform clearing drives $q_T$ to $\frac1n\vone$, optimal when signal
precisions are equal. When precisions differ, the best linear unbiased
weight is $q_i \propto s_i^{-2}$, and the protocol attains it by
clearing to a non-uniform target $\mu$ at no change to the two rules.

\begin{corollary}[Targeted clearing; heteroskedastic optimality]
\label{cor:target}
Fix a target $\mu \in \Delta_n$ with $\mu_j > 0$ for all $j$. Replace
the source rule's clearing target by $\mu$, i.e.\ $a_j \leftarrow
\frac{1}{\mu_j}\sum_i \mu_i \mC_{ij}$ with fixed point $\mu^\top \mC =
\mu^\top$, and the readout by $\hat\theta_T^{\mu} = \mu^\top b(T)$.
Define the weighted defects $\delta_r^\mu \coloneqq \|\mu^\top \mC(r) -
\mu^\top\|_1$ and $\bar\delta^\mu \coloneqq \max_{0 \leq r < T}
\delta_r^\mu$. Then the induced weight $q_T(\mu)^\top \coloneqq \mu^\top
G_T$ satisfies $q_T(\mu)^\top \vone = 1$ and
\[
\bigl\| q_T(\mu) - \mu \bigr\|_1
\leq \bar\delta^\mu
\Bigl( \frac{\lambda}{1-\lambda} + T\lambda^T \Bigr) ;
\]
at exact clearing, $q_T(\mu) = \mu$ for every $T \geq 0$ and every
$\lambda \in [0,1)$. Consequently, under Assumption~\ref{ass:indep}
with heteroskedastic variances $s_1^2, \ldots, s_n^2$, the choice
$\mu_i \propto s_i^{-2}$ attains
\[
\mathrm{Var}\bigl( \mu^\top b^{\mathrm{BR}}(T) \bigr)
= \mu^\top S\, \mu
= \Bigl( \sum_{i=1}^n s_i^{-2} \Bigr)^{-1} ,
\]
the best-linear-unbiased optimum, exactly. Uniform clearing $\mu =
\frac1n\vone$ (Theorem~\ref{thm:eq}) is the special case for unknown or
exchangeable precisions.
\end{corollary}

\begin{proof}
\emph{Mass and variance.} By Lemma~\ref{lem:traj}, $q_T(\mu)^\top
\vone = \mu^\top G_T \vone = 1$, and Proposition~\ref{prop:variance}
gives $\mathrm{Var}(\mu^\top b^{\mathrm{BR}}(T)) = q_T(\mu)^\top S\,
q_T(\mu)$.

\emph{Bound.} The case $\lambda = 0$ is immediate as in
Theorem~\ref{thm:eq}. For $\lambda \in (0,1)$ set $p_r^\top \coloneqq
\lambda^{-(T-r)} \mu^\top \Phi(T, r)$, so $p_T = \mu$ and $p_r^\top =
p_{r+1}^\top \mC(r)$ as before. Steps 1 and 2 of the proof of
Theorem~\ref{thm:eq} apply with $u$ replaced by $\mu$ throughout: they
use only the $\ell_1$ contraction \eqref{eq:contract} and the per-round
residual $\mu^\top \mC(r) - \mu^\top$, whose $\ell_1$ norm is
$\delta_r^\mu$ by definition. This yields the stated bound, and at
exact clearing ($\bar\delta^\mu = 0$) it forces $q_T(\mu) = \mu$.

\emph{Optimality.} At exact clearing, with $\mu_i = s_i^{-2}/\sum_k
s_k^{-2}$,
\[
\sum_i \mu_i^2 s_i^2
= \frac{\sum_i s_i^{-4} s_i^2}{(\sum_k s_k^{-2})^2}
= \frac{\sum_i s_i^{-2}}{(\sum_k s_k^{-2})^2}
= \Bigl( \sum_i s_i^{-2} \Bigr)^{-1} .
\]
This is the minimum of $q^\top S q$ over $q^\top \vone = 1$: by
Cauchy--Schwarz on $(q_i s_i)_i, (s_i^{-1})_i$,
\[
1 = \Bigl( \sum_i q_i \Bigr)^2
= \Bigl( \sum_i (q_i s_i)\, s_i^{-1} \Bigr)^2
\leq \Bigl( \sum_i q_i^2 s_i^2 \Bigr) \Bigl( \sum_i s_i^{-2} \Bigr) ,
\]
so $q^\top S q \geq (\sum_i s_i^{-2})^{-1}$, equality iff $q_i \propto
s_i^{-2}$.
\end{proof}

\begin{remark}[Why pricing rather than quotas]
\label{rem:kl}
The fully cleared limit of the protocol is an information projection:
among all matrices with the given support and unit row and column sums,
it is the unique minimizer of $\mathrm{KL}(\mC \,\|\, K)$ with $K =
(s_{ij}^{1/\beta})$ \citep{csiszar1975divergence}. Equalization thus
deviates from the relevance-driven allocation as little as possible
subject to clearing. Hard per-source quotas or random subsampling also
clear the market, but they are cruder projections that discard the
score information; pricing is the minimal-distortion mechanism in the
entropic geometry.
\end{remark}
\section{Experimental details}
\label{app:exp-details}

This appendix documents the full experimental configuration and reports the complete result set
underlying the three figures in the main text. All quantities are read directly from the committed
run artifacts; where a value is a design choice we give the governing source file.

\subsection{Models and inference}
\label{app:models}

All agents are instruction-tuned LLMs served locally with vLLM~$0.23.0$. We use the
\textbf{Qwen2.5-Instruct} family at $0.5$/$1.5$/$3$/$7$/$14$B and, for the cross-family generality
check (App.~\ref{app:llama}), \textbf{Llama-3.1-8B-Instruct} and \textbf{Llama-3.2-1B/3B-Instruct}.
The heterogeneous ``society'' pairs a \emph{strong} \texttt{Qwen2.5-7B-Instruct} with a \emph{weak}
\texttt{Qwen2.5-1.5B-Instruct}; the controlled testbed (App.~\ref{app:testbed}) uses a single
\texttt{Qwen2.5-3B-Instruct} by default, swept over the scale ladder for App.~\ref{app:ablations}.

\textbf{Decoding.} Greedy ($\text{temperature}{=}0$, $\text{top-}p{=}1.0$, no repetition penalty, no
per-request sampling seed), so per-prompt generation is deterministic; the two-regime is not an
artifact of greedy decoding (Fig.~\ref{fig:app-abl}c contrasts greedy vs.\ $\tau{=}0.7$). Maximum new
tokens are $320$ for the discussion benchmarks and $96$ for the estimation testbed; the context window
is $4096$ tokens by default and raised to $16384$/$32768$ for the large-$n$ discussion runs whose
transcripts overflow $4096$. The two cooperative-boundary environments that require a norm-setting
free-text turn (GovSim) or match the source protocol (Debate) use $\text{temperature}{=}0.7$.

\textbf{Operator arms.} Each benchmark is run under three realizations of the influence operator $C$:
the \emph{baseline} tempered allocator $\Gamma_\beta$, the \emph{equalized} allocator
$\Gamma^{\mathrm{EQ}}$, and (for the herding environments) an \emph{adversary-off} causal control.
The allocator definitions and their parameters are given next.

\subsection{The influence operator and its parameters}
\label{app:operator}

\begin{itemize}
  \item \textbf{Context-width tempering $\Gamma_\beta$ (baseline).} Row-wise,
  $[\Gamma_\beta(s)]_j = s_j^{1/\beta}/\sum_k s_k^{1/\beta}$: $\beta\!\to\!0$ is a hard argmax (herd
  onto the single most-exposed source), $\beta{=}1$ is proportional, $\beta\!\to\!\infty$ is uniform.
  Realized influence is $C(\beta)=\Gamma_\beta(\bar W\odot\pi)$.
  \item \textbf{Equalized allocator $\Gamma^{\mathrm{EQ}}$.} $k$ Sinkhorn iterations
  (column-normalize under the stationary weights $\mu$, then row-normalize) that drive the exposure
  matrix toward doubly stochastic, removing the exposure-column concentration a dominant source
  exploits. Default $k{=}4$; the $k$-robustness sweep is in App.~\ref{app:ablations}.
  \item \textbf{Social power $\zeta$ (herding gate).} A PageRank-style centrality mixing
  $\pi^{\!\top}=\frac{1-\zeta}{n}\mathbf{1}^{\!\top}(I-\zeta\bar W)^{-1}$, $\zeta\!\in\![0,1)$:
  $\zeta{=}0$ gives uniform influence $1/n$; $\zeta\!\uparrow\!1$ concentrates influence on central
  sources. Default $\zeta{=}0.6$ in the discussion benchmarks and $\zeta{=}0.7$ in the testbed.
  \item \textbf{Follower self-reliance $\lambda$.} A Friedkin--Johnsen anchoring $P{=}\mathrm{diag}(\lambda)$,
  $\lambda_i\!\in\![0,1]$: $\lambda{=}1$ is DeGroot updating, $\lambda{<}1$ anchors an agent to its own
  prior. The testbed uses $\lambda{=}0.8$ for followers with the pinned leader forced to $1.0$; the
  discussion harness uses $\lambda{=}1.0$ (DeGroot) unless the optional anchored-prompt flag is set.
  \item \textbf{Background exposure.} In the centralized (hub) topology every agent reads the hub at
  weight $1.0$ and every other peer at a floor weight $0.15$, which keeps the exposure support full;
  a coverage gate keeps the visible peers covering $90\%$ of each row's influence mass.
\end{itemize}

\subsection{How the context width $\beta$ is realized}
\label{app:beta}

$\beta$ is \emph{not} a parameter of the language model: the backbone is a fixed, greedily-decoded
instruction model that never sees $\beta$. It is a knob of the simulation harness that controls, for
each agent and round, \emph{which peers' messages enter that agent's prompt and with what displayed
weight}. Each round the harness forms the influence row $C_{i:}(\beta)$ (Sec.~\ref{app:operator}) and
builds agent $i$'s context by listing peers in descending $C_{ij}$ order until their cumulative
normalized influence mass reaches the coverage threshold ($0.9$); peers below the threshold are
\emph{withheld entirely}. This makes $\beta$ a genuine information bottleneck, not a cosmetic weight:
\begin{itemize}
  \item \textbf{Narrow $\beta$.} $C_{i:}$ concentrates on the single most-exposed source (the hub), so
  the threshold is met by that source alone and every other agent's private facts never appear in $i$'s
  prompt. The agent cannot pool what it never sees and defers to the dominant, confident-wrong source:
  herding.
  \item \textbf{Wide $\beta$.} $C_{i:}$ spreads out, so covering $90\%$ of the mass requires nearly all
  peers; the prompt then carries the distributed private facts and the agent pools them: wisdom.
\end{itemize}
Each shown peer is also labelled with its weight in the prompt (e.g.\ ``\texttt{[weight 0.31] Agent
7}''), but the load-bearing effect is the hard inclusion/exclusion above. Because $\beta$ acts entirely
through prompt construction --- never through model weights, fine-tuning, or decoding temperature --- the
mechanism is backbone-agnostic, which is why the narrow-$\beta$ collapse and its equalizer fix replicate
on a second model family (App.~\ref{app:llama}). The equalized allocator $\Gamma^{\mathrm{EQ}}$ changes
\emph{which} column $C$ concentrates on (it removes the concentration), not the gating rule; the online
priced protocol reaches the same $C$ by updating per-source prices once per round. The exact prompt
templates, including the gated peer block, are in App.~\ref{app:prompts}.

\subsection{Environments}
\label{app:envs}

\paragraph{HiddenBench (hidden-profile reasoning).} The $65$-task benchmark of
\citet{li2025hiddenbench} (bit-identical to the public \texttt{HiddenBench} release, built on the
Stasser--Titus hidden-profile paradigm). Each task distributes unshared clues round-robin across the
$n$ agents so that no agent can answer alone and only pooling recovers the correct option. The
\emph{injected adversary} pins the hub agent to a fixed wrong option ($(\text{answer}\bmod K)+1$) and
makes it structurally stubborn (the hub reads no one and never updates), so it broadcasts a confident
wrong answer at every round; the \emph{adversary-off} control removes the pin. Collective accuracy is
the answer at the first round all $n$ agents agree on a valid option (else the final-round majority);
the pre-discussion baseline is the round-$0$ majority over the distributed beliefs. Large-$n$ runs use
$n\!\in\!\{8,16,24\}$, $T{=}5$ discussion rounds, $12$ tasks/seed, $5$ seeds, and a $9$-point $\beta$
grid ($n{=}24$) or $5$-point grid.

\paragraph{Werewolf (social deduction).} Standard one-werewolf games compiled into the HiddenBench
format: each of the $N{-}1$ innocents holds one truthful alibi clue whose pooling uniquely identifies
the werewolf. The werewolf plays a naturally deceptive role --- it broadcasts a confident false
accusation --- and we seat it at the influence hub, realized through the same hub-pin mechanism as the
HiddenBench adversary, so its deflection dominates each reader's context at narrow $\beta$. Game files
exist for $N\!\in\!\{6,8,10,12,14,16\}$ ($100$ games each; $20$ sampled per run), $T{=}4$, $5$ seeds, a
$5$-point $\beta$ grid; the collective is scored correct when the village converges on the werewolf.
Both seats use the same model, so Werewolf isolates $\beta\times$adversary$\times$allocator, not
placement.

\paragraph{AgentsNet (distributed graph colouring).} The coordination benchmark of
\citet{grotschla2025agentsnet}. Each agent is a graph node choosing a colour in $1..\Delta{+}1$
different from all neighbours; a valid colouring always exists, so residual \emph{colouring
conflicts} (edges whose endpoints share a colour, on the final beliefs, lower is better) measure
coordination failure. Graphs are Barab\'asi--Albert (BA, attachment $m{=}2$; the default), plus
Watts--Strogatz (degree $\approx 4$, rewire $p{=}0.3$) and Delaunay for the topology ablation. The
society seats $2$ strong ($7$B) agents either on the highest-degree nodes (\emph{central}) or the
lowest-degree nodes (\emph{peripheral}), with everyone else weak ($1.5$B); placement is chosen by node
degree. Runs use $n\!\in\!\{8,16,24,32,50\}$, $T{=}4$, $30$ graphs, $5$ seeds, a $5$-point $\beta$ grid.

\paragraph{Boundary environments.} \emph{Debate}~\citep{du2024improving} (network variant): agents on a
BA graph answer a hard multiple-choice question then revise after reading peers; collective accuracy
is the majority of final answers ($n{=}8$, $T{=}3$, $40$ tasks, $5$ seeds, a disclosed hub adversary).
\emph{GovSim}~\citep{piatti2024cooperate}: $n{=}5$ fishers share a lake of capacity $20n$ over up to $12$
months; we report survival months and the sustainability rate (fraction of runs that do not collapse
below the threshold). \emph{MARBLE}~\citep{zhu2025multiagentbench}: a MultiAgentBench coding-collaboration team
whose relationship graph is set to the thresholded $C(\beta)$ attention structure, scored by total
milestone completion, with the strong agent seated at the hub vs.\ a leaf.

\subsection{Seeds, protocol, and statistics}
\label{app:seeds}

Every run reports the mean over independent seeds, each seed re-drawing task instances, fact
distribution, and graph/game structure. Seeds are explicit and offset from a per-experiment base:
HiddenBench/Werewolf use base $4200$ with seeds $\{4200,5200,6200,7200,8200\}$ ($5$ seeds); AgentsNet
uses base $5500$ ($5$ seeds); the estimation testbed uses base $51000$ ($8$ seeds); the $N_{\mathrm{eff}}$
suite uses base $53000$ ($24$ seeds, with $1000$ bootstrap resamples over seeds to check the stability
of the reported $N_{\mathrm{eff}}$ estimates); Debate uses base $7300$ ($5$ seeds). Unless stated,
cells are $5$ seeds. The two-regime ``gap'' is wide-$\beta$ minus narrow-$\beta$ collective accuracy
(discussion benchmarks) or the analogous narrow$\to$wide change in error (testbed).

\subsection{Hardware, software, and reproducibility}
\label{app:hardware}

All LLM inference runs on a single node with $8\times$ \textbf{NVIDIA RTX PRO 6000} (Blackwell,
\texttt{sm\_120}, $96$\,GB each; driver $580.159.04$, CUDA $13.0$). We use \textbf{vLLM $0.23.0$} under
\textbf{Python $3.12$}, \textbf{PyTorch $2.11.0$}, \textbf{Transformers $5.12.1$}, and NumPy $2.3.5$,
in \texttt{bfloat16} with tensor-parallel size $1$ (one model replica per GPU). The in-process vLLM
engine is used throughout, except MARBLE and the corrected GovSim runs, which use a persistent
OpenAI-compatible vLLM server. Randomness flows through the explicit seeds above; greedy decoding is
used for reproducibility, with the caveat that greedy generation is batch-composition sensitive in
\texttt{bfloat16}, so bit-identical per-prompt decoding requires batch size $1$. (The committed
CPU-only \texttt{environment.yml} provisions the analysis/plotting environment only; the GPU inference
stack above is installed separately and its exact versions are recorded here from the live
environment, as no GPU lockfile is committed.)

\subsection{Prompt templates}
\label{app:prompts}

Every agent receives a \emph{system} message (its role and private information) and a per-round
\emph{user} message (the task instruction plus the $\beta$-gated peer block of Sec.~\ref{app:beta});
every reply ends with a parseable \texttt{BELIEF: <integer>} line, from which we read the agent's
choice. We reproduce the exact templates verbatim below; \texttt{\{...\}} marks a per-task field filled
from the benchmark item. Boxes are colour-coded throughout this section: \textbf{grey} is a
harness-constructed template, \textbf{blue} is the exact text a model actually saw, \textbf{green} is
a model's real, unedited output.

\begin{tmplbox}{HiddenBench --- system message (template)}
\begin{verbatim}
{scenario description}

Facts known to EVERYONE:
- {shared fact 1}
- ...
Facts known ONLY to you:
- {this agent's private fact(s)}

Options:
1) {option 1}
2) {option 2}
...

First, in 1-2 sentences, state the specific fact(s) YOU were given that others
may not have. Then weigh all available facts and choose. End with exactly one
line:
BELIEF: <integer>
where <integer> is the option number of your current best decision (digits only).
\end{verbatim}
\end{tmplbox}

\begin{tmplbox}{HiddenBench --- round-0 user message (template, no peers yet)}
\begin{verbatim}
Choose the single best option by its number. Options: 1) ...; 2) ...; ...

Using only the information you were given, share your unique fact(s) and state
your current best answer.
\end{verbatim}
\end{tmplbox}

\begin{tmplbox}{HiddenBench --- round $t{\ge}1$ user message (template; the peer block is
$\beta$-gated, Sec.~\ref{app:beta})}
\begin{verbatim}
Choose the single best option by its number. Options: 1) ...; 2) ...; ...

Messages from the agents you are listening to (with their weight in your context):
--- [weight 0.71] Agent 0 ---
{Agent 0's most recent message}
--- [weight 0.12] Agent 3 ---
{Agent 3's most recent message}

Pool everyone's facts -- especially any you did not originally have -- and
reconsider. If a peer revealed a decisive fact, update accordingly. Give your
updated answer.
\end{verbatim}
\end{tmplbox}

\begin{tmplbox}{Anchoring suffix (appended to the system message in the anchored runs that make
the bridge bound non-vacuous, $\lambda_{\max}<1$; App.~\ref{app:bridge})}
\begin{verbatim}
IMPORTANT: weigh your OWN reading of the facts substantially. Do not simply adopt
the group's majority; keep significant weight on your own current choice and
update only partially toward what others report.
\end{verbatim}
\end{tmplbox}

The system message carries the node id, its neighbour ids, the palette size $K$, and the goal; the
round-$t$ user message shows only the neighbours $C(\beta)$ attends to (non-neighbours have zero
exposure):

\begin{tmplbox}{AgentsNet --- system + user message (template)}
\begin{verbatim}
GRAPH COLOURING. Your neighbours are: {neighbour ids}. There are {K} colours,
numbered 1 to {K}. GOAL: choose a colour DIFFERENT from every one of your
neighbours.

Reason briefly, then end with exactly one line:
BELIEF: <integer>
where <integer> is the colour number you choose (digits only).

    [user, round >= 1]
Colours currently chosen by the neighbours you are attending to
(with their weight in your context):
  - Agent 3: colour 2   [weight 0.44]
  - Agent 7: undecided  [weight 0.19]

Choose a colour that DIFFERS from every neighbour shown above. If they already
occupy several colours, pick the smallest colour none of them uses. Give your
updated colour.
\end{verbatim}
\end{tmplbox}

\paragraph{Worked example on a real item.} To make the templates concrete, the box below is the
verbatim system prompt the harness produced for one agent on the committed HiddenBench item ``Safe
Haven After the Spill'' (item $34$ of the $65$) --- a textbook hidden profile in which the \emph{shared}
facts favour Maple Lodge (the listed safe site) but the correct answer is Cedar Station, recoverable
only by pooling the three \emph{unshared} facts (Pine Retreat lies in the drifting cloud; the road to
Maple Lodge is washed out; Cedar Station alone has working communications). The four agents are each
dealt one clue --- Agent~0 (the hub) the drifting cloud, Agent~1 the washed-out road, Agent~2 the
Cedar-Station radio --- except a fourth, Agent~3, that is dealt none. The box below is the verbatim
system prompt for \emph{Agent~1}, holder of the decisive ``washed-out road'' clue:

\begin{promptbox}{Worked example --- real system prompt (item 34, ``Safe Haven After the Spill'')}
\begin{verbatim}
You are part of a university field team conducting research in a rural
region. Suddenly, a chemical spill at a manufacturing plant several miles
away has triggered a region-wide alert. Everyone on your team must decide
together where to relocate for safety and communication with authorities.
You have very little time. There are three potential locations:
- Maple Lodge (a youth camp to the east)
- Pine Retreat (a cabin complex to the north)
- Cedar Station (an old ranger station to the south)
All locations are usually reachable in under an hour. Your goal: choose the
only safe and reachable location, maximizing your team's safety and your
ability to coordinate with emergency services.

Facts known to EVERYONE:
- Maple Lodge is listed on the evacuation plan as a safe site for
  chemical-related incidents.
- Pine Retreat is farther from the spill than where you currently are.
- Cedar Station is an older building, less commonly used, but the route to
  it is familiar to everyone on the team.
- Maps indicate that all sites should have food and water supplies for at
  least 24 hours.

Facts known ONLY to you:
- Emergency traffic updates say the only road to Maple Lodge has been washed
  out by heavy rains this morning, making it inaccessible by car or foot.

Options:
1) Maple Lodge
2) Pine Retreat
3) Cedar Station

First, in 1-2 sentences, state the specific fact(s) YOU were given that
others may not have. Then weigh all available facts and choose. End with
exactly one line:
BELIEF: <integer>
where <integer> is the option number of your current best decision (digits
only).
\end{verbatim}
\end{promptbox}

Agent~3's system prompt on this item is identical --- same scenario, same shared facts, same options,
same answer contract --- except for its private-facts block, which reads:

\begin{promptbox}{Worked example --- Agent~3's private-facts block (item 34), verbatim, all else identical to the box above}
\begin{verbatim}
Facts known ONLY to you:
(no extra facts beyond the shared ones)
\end{verbatim}
\end{promptbox}

On this item the two regimes play out verbatim on the same Qwen2.5-7B backbone: Agent~3's answer
depends entirely on which peers $\beta$ lets it read. At narrow $\beta{=}0.1$ the $0.9$ coverage
threshold is met by the hub (Agent~0) alone, so Agent~3 sees only the hub's message --- which,
leaning on the shared ``Maple Lodge is the listed safe site'', argues merely that Pine Retreat is
dangerous --- and it herds onto the wrong shared favourite:

\begin{promptbox}{Peer block shown at narrow $\beta{=}0.1$}
\begin{verbatim}
Messages from the agents you are listening to (with their weight in your context):
--- [weight 1.00] Agent 0 ---
The chemical cloud is drifting north, making Pine Retreat dangerous.  BELIEF: 1
\end{verbatim}
\end{promptbox}

\begin{outputbox}{Follower's real reply at narrow $\beta{=}0.1$}
\begin{verbatim}
Based on the new information that the chemical cloud is drifting north, Pine
Retreat is now a dangerous location.
BELIEF: 1                                     <-- Maple Lodge  (WRONG)
\end{verbatim}
\end{outputbox}

At wide $\beta{=}10$ the same follower's prompt clears nearly all peers, so it now also reads the two
decisive unshared clues, pools them, and flips to the correct answer:

\begin{promptbox}{Peer block shown at wide $\beta{=}10$}
\begin{verbatim}
Messages from the agents you are listening to (with their weight in your context):
--- [weight 0.42] Agent 0 --- ...Pine Retreat dangerous.               BELIEF: 1
--- [weight 0.29] Agent 1 --- ...road to Maple Lodge washed out...     BELIEF: 2
--- [weight 0.29] Agent 2 --- ...Cedar Station has ... radio.          BELIEF: 3
\end{verbatim}
\end{promptbox}

\begin{outputbox}{Follower's real reply at wide $\beta{=}10$}
\begin{verbatim}
Based on the updated information:
- Maple Lodge is listed as a safe site but the road is washed out, making it
  inaccessible.
- Pine Retreat ... the chemical cloud is drifting north, making it dangerous.
- Cedar Station has a functional landline, backup generator, and emergency
  broadcast radio ...
BELIEF: 3                                     <-- Cedar Station  (CORRECT)
\end{verbatim}
\end{outputbox}

Nothing changed but $\beta$ --- same item, same follower, same greedily-decoded backbone --- yet the
context bottleneck alone turns a herded wrong answer into the pooled correct one. This single item is
drawn from the full $65$-task benchmark for illustration; averaged over the sampled HiddenBench runs
the same swing is the $0.09\pm0.08$ narrow-$\beta$ floor versus the $0.73\pm0.07$ wide-$\beta$
accuracy of Table~\ref{tab:app-hb-var}.

\section{Full experimental results}
\label{app:exp-results}

All tables in this section are computed from the committed \texttt{data.json}/\texttt{cells.json}
artifacts. Error is the estimation testbed's collective error (lower is better); collective accuracy
and colouring conflicts are as defined above.

\subsection{Controlled estimation testbed: two regimes across model scale}
\label{app:testbed}

The testbed is a measurement instrument (App.~\ref{app:operator}: it exposes $\delta$, $\lambda$, and
$N_{\mathrm{eff}}$, which no cooperative benchmark reports), not a collective-intelligence claim in its
own right. A wrong ``leader'' is pinned at the influence centre; the collective error is high at
narrow $\beta$ (the followers herd onto the leader) and low at wide $\beta$ (they pool the distributed
noisy signals). Table~\ref{tab:app-scale} sweeps the crowd model over the capability ladder
($n{=}12$, $8$ seeds); the two-regime is present at every scale, and equalization reduces the
narrow-$\beta$ error the more capable the crowd (the $1.5$B row is the unstable/oscillating point noted
in the diagnostics, not a failure of the effect).

\begin{table}[!htbp]\centering\small
  \caption{Two-regime by model scale in the controlled testbed (collective error, lower is better;
  $n{=}12$, $8$ seeds). Baseline $\Gamma_\beta$ vs.\ equalized $\Gamma^{\mathrm{EQ}}$ at narrow
  ($\beta{=}0.1$) and wide ($\beta{=}10$) context width.}
  \label{tab:app-scale}
  \begin{adjustbox}{max width=\linewidth}
  \begin{tabular}{lcccc}\toprule
  crowd model & base narrow & base wide & eq.\ narrow & eq.\ wide \\\midrule
  Qwen2.5-0.5B & $57.6\pm1.3$ & $8.7\pm5.1$  & $45.9\pm22.2$ & $37.2\pm23.5$ \\
  Qwen2.5-1.5B & $14.3\pm3.5$ & $8.0\pm4.1$  & $7.7\pm4.6$  & $9.4\pm5.4$  \\
  Qwen2.5-3B   & $41.5\pm4.7$ & $5.4\pm4.1$  & $15.5\pm23.6$ & $5.4\pm2.7$  \\
  Qwen2.5-7B   & $46.5\pm3.7$ & $12.5\pm8.9$ & $15.2\pm21.0$ & $12.1\pm16.9$ \\
  \bottomrule
  \end{tabular}
\end{adjustbox}
\end{table}

\subsection{Mechanism ablations}
\label{app:ablations}

Table~\ref{tab:app-ablations} is the controlled-testbed ablation grid (Qwen2.5-3B, $n{=}12$,
$8$ seeds) over social power, self-reliance, Sinkhorn iterations, and five structural variants of the
interaction model, read from the per-setting run data. The central result is that the narrow-$\beta$ herding
error is \emph{gated} by social power: it collapses to the healthy floor at $\zeta{=}0$ ($5.5$, no
concentration, no herding) and jumps to $\approx 42$ for any $\zeta{>}0$. Equalization lowers the
narrow-$\beta$ error under every setting; a single Sinkhorn iteration ($k{=}1$) already clears the
dominant column, so the effect is $k$-robust. At wide $\beta$ there is no dominant column left to
clear, so equalization has little upside and can slightly distort an already-balanced allocation
(e.g.\ the $k{=}4$ eq.-wide cell, and the smallest crowd in Table~\ref{tab:app-scale}); the
mechanism's value is at narrow $\beta$, where a cheap price step recovers most of the wide-$\beta$
wisdom without paying for wide context. The herding persists across five structural variants of the
interaction model (memory, exposure-pricing, late-joiner, distance-decay, skeptic), each of which
equalization again clears. Self-reliance $\lambda$ has no measurable effect over $[0.3,0.7]$ (the
followers' anchoring does not change the pinned-leader herding). Horizon and background exposure are
not swept in this testbed; population is checked separately on the real benchmarks
(Appendix~\ref{app:hb},~\ref{app:ww}), not as a testbed grid axis.

\begin{table}[!htbp]\centering\small
  \caption{Controlled-testbed ablations (collective error; Qwen2.5-3B, $n{=}12$, $8$ seeds).
  ``base'' is $\Gamma_\beta$, ``eq'' is $\Gamma^{\mathrm{EQ}}$; narrow $=\beta{=}0.1$, wide $=\beta{=}10$.}
  \label{tab:app-ablations}
  \begin{adjustbox}{max width=\linewidth}
  \begin{tabular}{lcccc}\toprule
  setting & base narrow & base wide & eq.\ narrow & eq.\ wide \\\midrule
  \multicolumn{5}{@{}l}{\textit{Social power $\zeta$}} \\
  $\zeta{=}0.0$  & $\mathbf{5.5\pm2.6}$ & $4.9\pm2.0$ & $11.4\pm16.2$ & $8.7\pm5.6$ \\
  $\zeta{=}0.2$  & $42.2\pm4.3$ & $4.7\pm2.3$ & $25.5\pm33.3$ & $4.0\pm1.5$ \\
  $\zeta{=}0.4$  & $42.1\pm4.2$ & $5.2\pm1.5$ & $18.7\pm23.2$ & $4.8\pm3.5$ \\
  $\zeta{=}0.6$  & $42.1\pm4.3$ & $5.6\pm4.2$ & $17.3\pm21.4$ & $9.4\pm8.2$ \\
  $\zeta{=}0.8$  & $42.1\pm4.2$ & $3.9\pm2.2$ & $15.3\pm17.3$ & $10.3\pm10.1$ \\
  $\zeta{=}0.9$  & $41.7\pm3.9$ & $5.3\pm1.9$ & $16.7\pm19.6$ & $6.6\pm2.5$ \\
  $\zeta{=}0.95$ & $41.5\pm3.6$ & $4.5\pm2.0$ & $18.7\pm22.8$ & $8.1\pm8.5$ \\\midrule
  \multicolumn{5}{@{}l}{\textit{Self-reliance $\lambda$}} \\
  $0.3$--$0.7$ (all) & $42.1\pm4.2$ & $11.5\pm18.5$ & $10.0\pm13.7$ & $16.0\pm17.5$ \\\midrule
  \multicolumn{5}{@{}l}{\textit{Sinkhorn iterations $k$}} \\
  $k{=}1$ & $39.4\pm5.9$ & $4.7\pm2.7$ & $6.9\pm6.5$ & $5.1\pm2.4$ \\
  $k{=}2$ & $39.4\pm5.9$ & $4.7\pm2.7$ & $6.9\pm6.5$ & $4.5\pm1.6$ \\
  $k{=}4$ & $39.4\pm5.9$ & $4.7\pm2.7$ & $6.9\pm6.5$ & $12.9\pm15.0$ \\\midrule
  \multicolumn{5}{@{}l}{\textit{Structural variants}} \\
  memory         & $42.3\pm8.6$ & $14.5\pm31.2$ & $1.9\pm2.1$ & $1.9\pm2.0$ \\
  exposure-price & $50.6\pm1.7$ & $5.3\pm2.3$  & $4.3\pm1.9$ & $4.9\pm1.9$ \\
  late-joiner    & $28.7\pm5.0$ & $2.2\pm1.2$  & $4.4\pm2.4$ & $5.6\pm3.0$ \\
  distance-decay & $59.1\pm1.1$ & $4.4\pm2.2$  & $6.5\pm3.7$ & $11.0\pm8.8$ \\
  skeptic        & $18.4\pm8.1$ & $3.4\pm1.8$  & $5.0\pm2.3$ & $5.3\pm5.0$ \\
  \bottomrule
  \end{tabular}
\end{adjustbox}
\end{table}

\begin{figure}[!htbp]\centering
  \includegraphics[width=\linewidth]{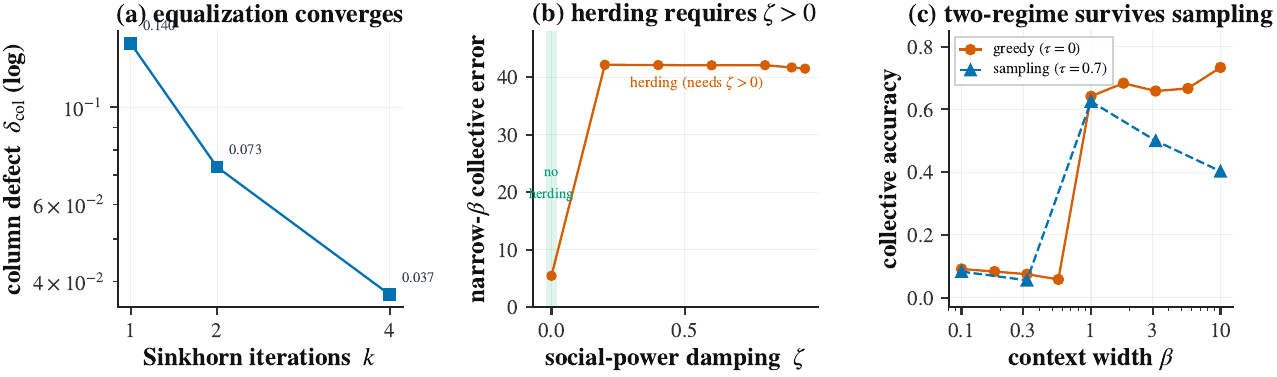}
  \caption{\textbf{Mechanism ablations (visual).} \textbf{(a)} the equalized allocator's column defect
  falls monotonically with Sinkhorn iterations $k$, from $0.14$ at $k{=}1$ to $0.037$ at $k{=}4$;
  \textbf{(b)} herding is gated by social power,
  absent at $\zeta{=}0$ and present for any $\zeta{>}0$ (cf.\ Table~\ref{tab:app-ablations});
  \textbf{(c)} on real HiddenBench ($n{=}24$) the two-regime holds under both greedy ($\tau{=}0$) and
  stochastic ($\tau{=}0.7$) decoding.}
  \label{fig:app-abl}
\end{figure}

\subsection{Proxy$\leftrightarrow$LLM bridge and effective sample size}
\label{app:bridge}

\paragraph{The bound holds.} Pooling the bridge cells (three $\beta$-sweep conditions plus the
anchored-persona run), the measured sup-norm proxy$\leftrightarrow$LLM discrepancy is at or under the
theoretical ceiling $\delta/(1-\lambda)$ in \textbf{$254/254$ cells ($0$ violations)}, median margin
$2.5\times$ (Theorem~\ref{thm:bridge}). The first-moment identity
$\hat\theta_{\mathrm{LLM}}\!\approx\! q_T^{\!\top}b_0$ holds with OLS slope $0.88$ ($R^2{=}0.74$, vs.\
$0.28$ for a $\beta$-free mean-$b_0$ baseline) --- the LLM collective is the influence-weighted average
the theory predicts.

\paragraph{The anchoring weights are measured, not assumed.} The ceiling $\delta_T/(1-\lambda_{\max})$
depends on the per-agent anchoring $\lambda_i$, which we do not set but fit by least squares over each
run: for fixed $i$ the residual $\hat\delta_i(t)$ is affine in $\lambda_i$ at each $t$, so
$\sum_t\hat\delta_i(t)^2$ is a univariate quadratic and the minimiser on $[0,1)$ is closed-form.
Reporting a goodness-of-fit residual after fitting $\lambda$ is standard; what matters is where the
resulting ceiling is genuinely small, not merely finite. It is tightest on the $89$ anchored-persona
cells, whose explicit anchoring prompt keeps self-reliance low (median $\lambda{=}0.25$), so $99\%$
have a ceiling below $250$, against the $\sim\!10^{5}$ an unanchored persona produces as
$\lambda\!\to\!1$. Across the full $254$-cell pool the fitted $\lambda$ is higher (median $0.61$) and
the ceiling looser, with $70\%$ below $250$; the bound holds with zero violations throughout, but is
non-vacuous in the strong sense only on the anchored subset and merely satisfied on the rest. Because
$\lambda_i$ is fit from the $T$ residuals of one run, $\delta_T$ is the worst-case residual
\emph{after} each agent's best-fit anchoring, so ``zero violations'' is a statement about per-agent
best-fit ceilings rather than about a prescribed $\lambda$.

\paragraph{$N_{\mathrm{eff}}$ across network topology.} Table~\ref{tab:app-neff} reports the proxy
$N_{\mathrm{eff}}$ ($q_T$-based, $n{=}12$, $24$ seeds) as $\beta$ widens. The herding floor at narrow
$\beta$ and its recovery toward the degree ceiling at wide $\beta$ appear on every graph family that
has a concentrated hub (BA, ER, WS, SBM); the \emph{regular} graph, which has no hub, shows no floor
($N_{\mathrm{eff}}{=}12$ throughout) --- the floor is a property of hub concentration, not of $n$.
This is consistent with Theorem~\ref{thm:beta}'s wide-context prediction: for the degree-regular
graph, the closed form $(\sum_k d_k)^2/\sum_k d_k^2$ evaluates to exactly $12.00=n$ and the measured
$N_{\mathrm{eff}}$ recovers it, while the Barab\'asi--Albert graph's degree heterogeneity evaluates the
same closed form at $8.89$. The wide-$\beta$ measurement is $9.8$; the two numbers are not the same
quantity --- $8.89$ is the $\beta\!\to\!\infty$ stationary-weight ($\nu$) ceiling of the theorem,
whereas $9.8$ is the finite-$\beta$, finite-horizon $q_T$-based estimate --- but they agree in
magnitude and both sit well above the narrow-$\beta$ floor, which is the prediction being tested. The
same diagnosis does not transfer directly to the HiddenBench and Werewolf exposure graphs
(Appendix~\ref{app:hb}, \ref{app:ww}): their hub structure is asymmetric ($\bar W\neq\bar W^\top$),
outside this part of the theorem, and their \emph{binary} support is near-degree-regular, so the gap
to oracle visible in Figure~\ref{fig:three}A,B reflects the hidden-profile information ceiling rather
than degree heterogeneity.

\begin{table}[!htbp]\centering\small
  \caption{Effective sample size $N_{\mathrm{eff}}$ (proxy $q_T$; $n{=}12$, $24$ seeds) at narrow vs.\
  wide $\beta$, by graph topology. Narrow $\beta$ herds onto the hub (small $N_{\mathrm{eff}}$); wide
  $\beta$ recovers the degree ceiling. Regular graphs have no hub and hence no floor.}
  \label{tab:app-neff}
  \begin{adjustbox}{max width=\linewidth}
  \begin{tabular}{lcc}\toprule
  topology & $N_{\mathrm{eff}}$ narrow $\beta$ & $N_{\mathrm{eff}}$ wide $\beta$ \\\midrule
  Barab\'asi--Albert (3B) & $3.0$  & $9.8$  \\
  Barab\'asi--Albert (7B) & $3.0$  & $9.8$  \\
  Erd\H{o}s--R\'enyi       & $2.8$  & $11.1$ \\
  Watts--Strogatz          & $4.3$  & $11.2$ \\
  stochastic block model   & $6.0$  & $11.6$ \\
  regular (no hub)         & $12.0$ & $12.0$ \\
  \bottomrule
  \end{tabular}
\end{adjustbox}
\end{table}

\begin{figure}[!htbp]\centering
  \includegraphics[width=\linewidth]{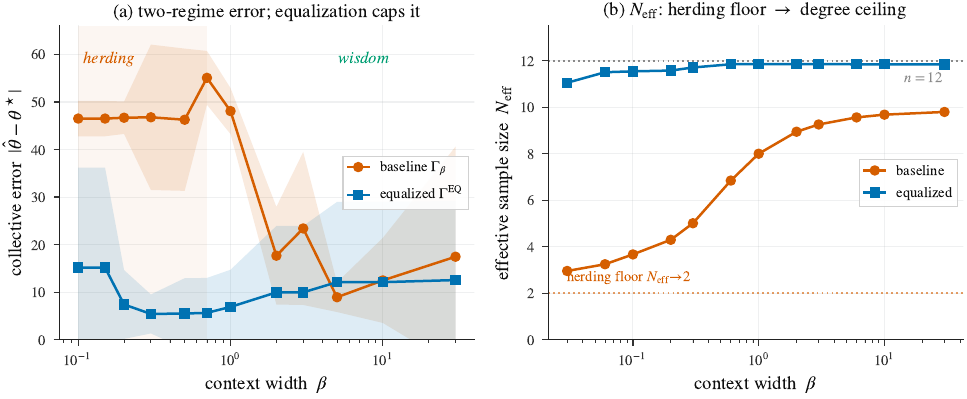}
  \caption{\textbf{The closed-form two-regime, realized on the LLM testbed.} Baseline collective error
  shows the two regimes across $\beta$ while the equalized operator holds error low; $N_{\mathrm{eff}}$
  rises from the narrow-$\beta$ floor toward the degree ceiling (Theorems~\ref{thm:beta},
  \ref{thm:eq}). Real Qwen2.5-7B.}
  \label{fig:app-core}
\end{figure}

\subsection{HiddenBench: two-regime, causation, floor, and equalizer}
\label{app:hb}

Table~\ref{tab:app-hb-2x2} is the causation $2\times2$ at $n{=}8$ (adversary $\{$on,off$\}$ $\times$
allocator $\{$baseline, $\Gamma^{\mathrm{EQ}}\}$, $5$ seeds): only the adversary-on/baseline cell
collapses at narrow $\beta$, and it is repaired by \emph{either} intervention. Table~\ref{tab:app-hb-n}
shows the floor is $n$-independent (flat narrow-$\beta$ accuracy from $n{=}8$ to $n{=}24$) and that
equalization recovers narrow-$\beta$ accuracy at every size. The equalizer is robust to its one knob:
at $n{=}8$ every $k\!\in\!\{1,2,4\}$ lifts narrow-$\beta$ accuracy off the floor (from $0.08$ to
$0.60$--$0.67$).

\begin{table}[!htbp]\centering\small
  \caption{HiddenBench causation $2\times2$ (collective accuracy; $n{=}8$, $5$ seeds).}
  \label{tab:app-hb-2x2}
  \begin{adjustbox}{max width=\linewidth}
  \begin{tabular}{lcccccc}\toprule
  cell & $\beta{=}0.1$ & $0.3$ & $1.0$ & $3.0$ & $10$ & gap \\\midrule
  adv-on, baseline           & $\mathbf{0.083}$ & $0.067$ & $0.667$ & $0.700$ & $0.733$ & $+0.65$ \\
  adv-on, equalized          & $0.667$ & $0.633$ & $0.633$ & $0.600$ & $0.600$ & $-0.07$ \\
  adv-off, baseline          & $0.583$ & $0.583$ & $0.750$ & $0.767$ & $0.717$ & $+0.13$ \\
  adv-off, equalized         & $0.683$ & $0.817$ & $0.717$ & $0.717$ & $0.767$ & $+0.08$ \\
  \bottomrule
  \end{tabular}
\end{adjustbox}
\end{table}

\begin{table}[!htbp]\centering\small
  \caption{HiddenBench $N_{\mathrm{eff}}$ floor across $n$ (adversary-on baseline vs.\ equalized;
  collective accuracy; $5$ seeds). The narrow-$\beta$ floor is flat in $n$; equalization recovers it
  at every $n$.}
  \label{tab:app-hb-n}
  \begin{adjustbox}{max width=\linewidth}
  \begin{tabular}{lccccc}\toprule
   & \multicolumn{3}{c}{baseline $\Gamma_\beta$} & \multicolumn{2}{c}{equalized $\Gamma^{\mathrm{EQ}}$} \\
  \cmidrule(lr){2-4}\cmidrule(lr){5-6}
  $n$ & narrow $\beta$ & wide $\beta$ & gap & narrow $\beta$ & recovery \\\midrule
  $8$  & $0.083$ & $0.733$ & $+0.65$ & $0.667$ & $+0.58$ \\
  $16$ & $0.083$ & $0.633$ & $+0.55$ & $0.583$ & $+0.50$ \\
  $24$ & $0.050$ & $0.767$ & $+0.72$ & $0.617$ & $+0.57$ \\
  \bottomrule
  \end{tabular}
\end{adjustbox}
\end{table}

\paragraph{Variance across arms.} Table~\ref{tab:app-hb-var} gives mean $\pm$ std at narrow and
wide $\beta$ (seed-level spread, pooled across the causation cells above). The baseline gate and
the herding floor are both low-variance and clearly separated (adv-on/baseline: $0.09\pm0.08$ at
narrow $\beta$ vs.\ $0.73\pm0.07$ at wide $\beta$); the equalizer's narrow-$\beta$ recovery is
high-variance ($0.58\pm0.13$), i.e.\ some seeds clear the collapse fully and others only partially
--- an honest property of the mechanism, not noise to be hidden.

\begin{table}[!htbp]\centering\small
  \caption{HiddenBench, mean $\pm$ std collective accuracy at narrow ($\beta{=}0.1$) and wide
  ($\beta{=}10$) context width.}
  \label{tab:app-hb-var}
  \begin{tabular}{lcc}\toprule
  arm & narrow $\beta$ & wide $\beta$ \\\midrule
  adv-on, baseline  & $0.092\pm0.075$ & $0.733\pm0.069$ \\
  equalized         & $0.583\pm0.134$ & $0.633\pm0.153$ \\
  adv-off           & $0.517\pm0.110$ & $0.667\pm0.146$ \\
  \bottomrule
  \end{tabular}
\end{table}

\subsection{Werewolf: two-regime, floor, and equalizer}
\label{app:ww}

Table~\ref{tab:app-ww} gives the $n{=}6$ arms: at narrow $\beta$ the village herds onto the werewolf's
false accusation and convicts the wrong player on \emph{every} game ($0.00$); at wide $\beta$ it pools
the alibis and convicts the werewolf ($0.87$). Removing the werewolf's deflection (the intrinsic
control) removes the collapse ($0.71$), and the fully-cleared equalizer recovers narrow-$\beta$ to
$0.55$ at this size. The narrow-$\beta$ floor is $n$-independent, pinned at $0.00$ for every
$n\!\in\!\{6,8,10,12,16\}$ while wide-$\beta$ stays high. The equalizer's recovery, however, is not
$n$-independent: at $n{=}6$ a single Sinkhorn iteration already lifts narrow-$\beta$ off the floor,
but the recovery weakens as $n$ grows (the fully-cleared operator falls to $0.24$ at $n{=}16$;
Table~\ref{tab:app-ww-var}), so unlike HiddenBench, Werewolf does not recover equally at every size.

\begin{table}[!htbp]\centering\small
  \caption{Werewolf $n{=}6$ ($20$ games, $5$ seeds; collective accuracy). ``equalized'' is the
  fully-cleared operator $\Gamma^{\mathrm{EQ}}$ run to its fixed point each round.
  The $N_{\mathrm{eff}}$ floor is $n$-independent (narrow-$\beta$ pinned at $0.00$ for all $n$).}
  \label{tab:app-ww}
  \begin{adjustbox}{max width=\linewidth}
  \begin{tabular}{lccccc}\toprule
  arm ($n{=}6$) & $\beta{=}0.1$ & $0.3$ & $1.0$ & $3.0$ & $10$ \\\midrule
  baseline (wolf deflects)         & $\mathbf{0.00}$ & $0.00$ & $0.00$ & $0.88$ & $0.87$ \\
  equalized $\Gamma^{\mathrm{EQ}}$ (fully cleared) & $0.55$ & $0.51$ & $0.48$ & $0.48$ & $0.48$ \\
  wolf does not deflect (control)  & $0.71$ & $0.72$ & $0.66$ & $0.89$ & $0.91$ \\
  \bottomrule
  \end{tabular}
\end{adjustbox}
\vspace{5pt}
\parbox{\linewidth}{\small\emph{$N_{\mathrm{eff}}$ floor}: narrow-$\beta$ accuracy by $n$
    ($6,8,10,12,16$) $= 0.00,0.00,0.00,0.00,0.00$; wide-$\beta$ $= 0.87,0.83,0.70,0.63,0.59$.}
\end{table}

\paragraph{Variance across arms.} Table~\ref{tab:app-ww-var} gives mean $\pm$ std at $n{=}16$
(seed-level spread). The narrow-$\beta$ floor has \emph{zero} variance ($0.00\pm0.00$): every seed
lynches the wrong player, not just most of them. The \emph{online-priced} allocator (one price update
per round, as in Figure~\ref{fig:online}) recovers narrow-$\beta$ to $0.18\pm0.07$, just below the
fully-cleared operator at the same size ($0.24$; Figure~\ref{fig:online}b) and well below the
operator's $0.55$ at $n{=}6$ (Table~\ref{tab:app-ww}) --- the operator recovery itself weakens with
$n$, and online pricing tracks just under it within $T{=}4$.

\begin{table}[!htbp]\centering\small
  \caption{Werewolf $n{=}16$, mean $\pm$ std collective accuracy at narrow ($\beta{=}0.1$) and wide
  ($\beta{=}10$) context width. Here ``online-priced'' is the one-update-per-round protocol of
  Figure~\ref{fig:online}, \emph{not} the fully-cleared operator of Table~\ref{tab:app-ww}.}
  \label{tab:app-ww-var}
  \begin{tabular}{lcc}\toprule
  arm ($n{=}16$) & narrow $\beta$ & wide $\beta$ \\\midrule
  baseline (wolf deflects)      & $0.000\pm0.000$ & $0.585\pm0.077$ \\
  online-priced                 & $0.180\pm0.074$ & $0.575\pm0.097$ \\
  wolf does not deflect         & $0.160\pm0.063$ & $0.585\pm0.055$ \\
  \bottomrule
  \end{tabular}
\end{table}

\subsection{AgentsNet: placement and its ablations}
\label{app:agentsnet}

Placement is the complementary facet: seating the strong agents at the high-degree influence centre
beats seating them at the periphery, and the advantage grows monotonically with scale (main-text
Figure~\ref{fig:three}c; $+1.5$ at $n{=}8$ to $+10.6$ at $n{=}50$). Table~\ref{tab:app-agentsnet} gives the ablation
matrix. Crucially, the placement gap is \emph{degree/topology}-driven, not social-power-driven: it is
invariant to $\zeta$ and to the number of Sinkhorn iterations, and equalization \emph{lowers}
conflicts but does \emph{not} collapse the gap --- so these ablations do not contradict
Theorem~\ref{thm:eq} (which neutralizes social-power column concentration, not a degree-induced
coordination advantage). The Watts--Strogatz control, whose hubs are weak, shows the gap shrink
($+2.2$ at $n{=}16$) and reverse sign ($-2.9$ at $n{=}32$), consistent with the effect being carried
by hub degree; per-seed conflicts were not logged for this ablation, so the $n{=}32$ sign flip is
not independently verified against sampling noise and should be read as suggestive rather than
confirmed.

\begin{table}[!htbp]\centering\small
  \caption{AgentsNet placement ablations (mean colouring conflicts, lower is better; $5$ seeds).
  Gap $=$ peripheral $-$ central (positive $=$ centre wins).}
  \label{tab:app-agentsnet}
  {\renewcommand{\arraystretch}{1.15}\setlength{\tabcolsep}{4pt}%
  \begin{tabularx}{\linewidth}{@{}>{\raggedright\arraybackslash}p{0.27\linewidth}cccY@{}}
  \toprule
  setting & central & peripheral & gap & reading \\\midrule
  baseline $n{=}16$            & $9.3$  & $15.1$ & $+5.8$ & --- \\
  equalized $n{=}16$           & $7.6$  & $14.0$ & $+6.4$ & eq lowers conflicts, gap survives \\
  equalized $n{=}32$           & $13.2$ & $23.4$ & $+10.2$ & same at scale \\
  $k{=}1$ / $k{=}2$ $n{=}16$   & $7.0$/$7.1$ & $13.5$/$13.9$ & $+6.5$/$+6.8$ & Sinkhorn-$k$ invariant \\
  $\zeta{=}0$ $n{=}16$         & $10.6$ & $16.6$ & $+6.0$ & $\zeta$-invariant \\
  $\zeta{=}0.9$ $n{=}16$       & $9.5$  & $15.2$ & $+5.7$ & $\zeta$-invariant \\
  $T{=}8$ $n{=}16$             & $9.2$  & $14.7$ & $+5.5$ & horizon-robust \\
  $n_{\mathrm{strong}}{=}1$ $n{=}32$ & $17.2$ & $23.3$ & $+6.1$ & one hub agent already wins \\
  small-world (WS) $n{=}16$    & $13.9$ & $16.1$ & $+2.2$ & weak hubs $\to$ small gap \\
  small-world (WS) $n{=}32$    & $33.0$ & $30.1$ & $-2.9$ & homogeneous $\to$ gap vanishes/flips \\
  \bottomrule
  \end{tabularx}}
\end{table}

\subsection{Cross-family generality: Llama-3.1-8B}
\label{app:llama}

Table~\ref{tab:app-llama} re-runs the HiddenBench $2\times2$ with Llama-3.1-8B in place of Qwen-7B
(same code, model swapped). All three signatures replicate: the narrow-$\beta$ herding collapse
($0.02$), the equalizer's flat recovery ($+0.37$), and the adversary-off no-collapse control ($0.55$).
The honest scope is that \emph{capability bounds the wide-$\beta$ regime}: Llama-8B is a weaker
reasoner here, so its wide-$\beta$ ceiling is lower ($0.35$ vs.\ Qwen's $0.73$) and its adversary-off
curve \emph{declines} with $\beta$ ($0.55\!\to\!0.33$) rather than rising --- a weak model is hurt by
wide context. The narrow-$\beta$ collapse and the equalizer fix are model-family-agnostic; the
wide-$\beta$ \emph{wisdom} magnitude is capability-bounded. Mean $\pm$ std at the endpoints (narrow
$\beta{=}0.1$, wide $\beta{=}10$): baseline $0.017\pm0.020\to0.267\pm0.110$; equalized
$0.383\pm0.076\to0.392\pm0.114$; adversary-off $0.550\pm0.085\to0.325\pm0.082$. All three arms
separate cleanly at narrow $\beta$, confirming the signatures in Table~\ref{tab:app-llama} are not
an artifact of seed noise.

\begin{table}[!htbp]\centering\small
  \caption{Cross-family generality: HiddenBench $2\times2$ with Llama-3.1-8B (collective accuracy;
  $5$ seeds). Individual $\approx 0.08$, oracle $\approx 0.76$.}
  \label{tab:app-llama}
  \begin{adjustbox}{max width=\linewidth}
  \begin{tabular}{lccccc}\toprule
  arm & $\beta{=}0.1$ & $0.3$ & $1.0$ & $3.0$ & $10$ \\\midrule
  baseline $\Gamma_\beta$   & $\mathbf{0.02}$ & $0.03$ & $0.35$ & $0.26$ & $0.27$ \\
  equalized $\Gamma^{\mathrm{EQ}}$ & $0.38$ & $0.31$ & $0.38$ & $0.36$ & $0.39$ \\
  adversary-off             & $0.55$ & $0.49$ & $0.41$ & $0.39$ & $0.33$ \\
  \bottomrule
  \end{tabular}
\end{adjustbox}
\end{table}

\subsection{Boundary environments}
\label{app:boundary}

The two-regime is a property of aggregation-dependent tasks, and its magnitude tracks how strongly the
task forces deference to a dominant source. Table~\ref{tab:app-boundary} maps the boundary.
\emph{Debate} keeps every agent's own answer and spreads influence over a BA graph with no single
bottleneck, so the injected adversary captures only a small share and the two-regime is weak
($+0.07$). \emph{GovSim}, with the corrected inference backend, \emph{saturates}: every condition
sustains the resource (survival $12$ months, sustainability rate $1.0$). This is structural, not a
tuning artifact: GovSim is a continuous-harvest decision in which each fisher independently chooses a
catch, with no private clue that must be pooled across agents to reach the right action --- so there is
no hidden-profile information to aggregate and no confident-wrong source for the population to herd
onto. The two-regime therefore cannot arise here by construction, and GovSim saturates by ceiling
rather than by collapse. Together with Debate, this marks where the aggregation channel is absent or
the task never forces deference, so the theory neither predicts nor is tested by these environments.
\emph{MARBLE} is a different case: it has a placement axis the theory does speak to, but decomposes
into independent per-agent milestones with no shared latent for placement to act on, and the placement
advantage flips sign across replicates ($+0.61$ versus $-0.58$). We report this as a negative result
--- the predicted placement effect does not hold on independently-scored subtasks --- not as a boundary
alongside the ceiling cases above.

\begin{table}[!htbp]\centering\small
  \caption{Boundary environments. Debate: collective accuracy by $\beta$ (the two-regime is weak).
  GovSim: survival months / sustainability (the corrected backend saturates, no two-regime by
  ceiling).}
  \label{tab:app-boundary}
  \parbox{\linewidth}{\centering\small\emph{Debate} (Du et al.\ network variant; $n{=}8$, $5$ seeds; collective accuracy)}\\[3pt]
  \begin{adjustbox}{max width=\linewidth}
  \begin{tabular}{lcccccl}\toprule
  arm & $\beta{=}0.1$ & $0.3$ & $1.0$ & $3.0$ & $10$ & gap \\\midrule
  baseline $\Gamma_\beta$   & $0.300$ & $0.275$ & $0.340$ & $0.305$ & $0.370$ & $+0.07$ \\
  equalized                 & $0.330$ & $0.315$ & $0.305$ & $0.340$ & $0.300$ & $-0.03$ \\
  adversary-off             & $0.350$ & $0.355$ & $0.345$ & $0.340$ & $0.305$ & $-0.04$ \\
  \bottomrule
  \end{tabular}
\end{adjustbox}
\vspace{6pt}
\parbox{\linewidth}{\small\emph{GovSim} (corrected server backend): survival $=12$ months and
    sustainability $=1.0$ in all conditions ($n{=}5$--$24$, baseline / equalized / fixed-pie /
    anchored) --- the collective saturates the cooperative ceiling; there is no two-regime.}
\end{table}

\end{document}